\documentclass[a4paper,12pt]{article} 

\usepackage{times}
\usepackage{soul}
\usepackage{url}
\usepackage[hidelinks]{hyperref}
\usepackage[utf8]{inputenc}
\usepackage[small]{caption}
\usepackage{graphicx}
\usepackage{amsthm}
\usepackage{booktabs}
\usepackage{algorithm}
\usepackage{algorithmic}
\usepackage{subfig}
\usepackage{microtype}
\usepackage{xspace}
\usepackage{amsmath}

\usepackage[fixamsmath,disallowspaces]{mathtools}

\usepackage{microtype}

\usepackage{comment}
\usepackage{macros}
\usepackage{multicol}
\usepackage{multirow}

\newtheorem{theorem}{Theorem}
\newtheorem{proposition}[theorem]{Proposition}
\newtheorem{corollary}[theorem]{Corollary}
\newtheorem{example}[theorem]{Example}

\usepackage{enumerate}
\usepackage{xcolor}

\usepackage{tikz}
\usetikzlibrary{arrows}
\usepackage{todonotes}

\usepackage{array}
\newcolumntype{H}{>{\setbox0=\hbox\bgroup}c<{\egroup}@{}}

\newcommand{\SB}{\{}%
\newcommand{\SM}{\mid}%
\newcommand{\SE}{\}}%

\newcommand{\complexityClassFont}[1]{\ensuremath{\mathrm{#1}}}

\newcommand{\numberDotP}{\complexityClassFont{\#{\cdot}\Ptime}\xspace}
\newcommand{\numberDotCoNP}{\complexityClassFont{\#\cdot\co\NP}\xspace}

\usepackage{thmtools}
\usepackage{thm-restate}
\usepackage{cleveref}

\usepackage{apptools}
\usepackage{xpatch}
\makeatletter
\xpatchcmd{\thmt@restatable}
{\csname #2\@xa\endcsname\ifx\@nx#1\@nx\else[{#1}]\fi}
{\IfAppendix{\csname #2\@xa\endcsname}{\csname #2\@xa\endcsname\ifx\@nx#1\@nx\else[{#1}]\fi}}
{}{} 
\makeatother
\usepackage{authblk}

\usepackage{pgf,tikz}
\usetikzlibrary{fit,positioning,arrows,shapes, decorations,automata,%
	petri,%
	topaths,calc}
\tikzstyle{dashedarrow} = [-stealth',dashed]
\tikzstyle{tdnode} = [draw,rounded corners,top color=vertexTopColor,bottom color=vertexBottomColor,minimum size=1.5em]
\tikzstyle{stdnode} = [tdnode, font=\scriptsize]
\tikzstyle{stdnodecompact} = [stdnode, inner sep = 1.5pt, outer sep = 0.1pt]
\tikzstyle{stdnodetable} = [stdnode, inner sep = 1.5pt, outer sep = 0]
\tikzstyle{stdnodenum} = [minimum size=1.5em, font=\scriptsize]
\tikzstyle{tdedge} = [-,draw,thick]
\tikzstyle{tdlabel} = [draw=none, rectangle, fill=none, inner sep=0pt, font=\scriptsize]

\newcommand{\futuresketch}[1]{}
\newcommand{\poly}[0]{poly}

\DeclareMathOperator{\var}{\mathsf{var}}
\DeclareMathOperator{\matr}{\mathsf{matrix}}
\DeclareMathOperator{\tower}{\ensuremath{\mathsf{exp}}}

\DeclareMathOperator{\children}{\mathsf{chldn}}

\newcommand{\eqdef}{\ensuremath{\,\mathrel{\mathop:}=}}

\newcommand{\ptlong}{\ensuremath{\mathcal{X}}}
\newcommand{\ptshort}{\ensuremath{\mathcal{X}}}

\usepackage{xcolor}
\colorlet{vertexTopColor}{white}

\colorlet{vertexBottomColor}{black!10}

\newcommand{\cw}[1]{\mathsf{cw}(#1)}

\newcommand{\sinc}[1]{\mathcal{G}^d_i{(#1)}}

\allowdisplaybreaks

\DeclareMathOperator{\adef}{def}

\newcommand{\conf}{\mathsf{conf}}

\newcommand{\sigmai}{\mathsf{semiSt}}
\newcommand{\colors}{\mathsf{cols}}
\newcommand{\child}{\mathsf{child}}

\newcommand{\lefttriangle}{\hfill$\triangleleft$}
\AtEndEnvironment{claimproof}{\filledsquare}
\AtEndEnvironment{example}{\lefttriangle}

\usepackage{enumerate,enumitem}

\newcommand{\postrebuttal}[1]{\textcolor{black}{#1}}

\newcommand{\co}{\complexityClassFont{co}}

\DeclareMathOperator{\scw}{\mathsf{dcw}}
\DeclareMathOperator{\BigO}{\mathcal{O}}
\newcommand{\Card}[1]{\left|#1\right|}
\newcommand{\CCard}[1]{\|#1\|}
\newcommand{\ta}[1]{\ensuremath{2^{#1}}}

\newcommand{\col}{\ensuremath{\text{col}}}
\renewcommand{\phi}{\varphi}
\newcommand{\RNum}[1]{\uppercase\expandafter{\romannumeral #1\relax}}

\makeatletter
\providecommand*{\cupdot}{%
	\mathbin{%
		\mathpalette\@cupdot{}%
	}%
}
\newcommand*{\@cupdot}[2]{%
	\ooalign{%
		$\m@th#1\cup$\cr
		\hidewidth$\m@th#1\cdot$\hidewidth
	}%
}
\makeatother
\DeclareMathOperator{\dcup}{\cupdot}

\def\hy{\hbox{-}\nobreak\hskip0pt}

\begin{document} 

\title{Structure-Aware Encodings of Argumentation Properties for Clique-width\footnote{Authors are stated in reverse alphabetical order.}}

\author[1]{Yasir Mahmood\thanks{\texttt{yasir.mahmood@uni-paderborn.de}}}
\author[2]{Markus Hecher\thanks{\texttt{hecher@mit.edu}}}
\author[3]{Johanna Groven\thanks{\texttt{johanna.groven@liu.se}}}
\author[3]{Johannes K. Fichte\thanks{\texttt{johannes.fichte@liu.se}}}

\affil[1]{Data Science Group, Heinz Nixdorf Institute, Paderborn University, Germany}
\affil[2]{CNRS, UMR 8188, Centre de Recherche en Informatique de Lens (CRIL), University of Artois, France}
\affil[3]{Department of Computer and Information Science, Link\"oping University, Sweden}

\date{\vspace{-5ex}}  

\maketitle

\begin{abstract}
	Structural measures of graphs, such as treewidth, are central tools
	in 
	computational complexity resulting in efficient algorithms when
	exploiting the parameter.
	It is even known that modern SAT solvers work efficiently on
	instances of small treewidth.
	Since these solvers
	are widely applied, 
	research interests in compact 
	encodings into (Q)SAT for solving and to understand encoding
	limitations.
	%
	%
	%
	%
	Even more general is the graph parameter clique-width, which unlike
	treewidth can be small for dense graphs. 
	Although
	algorithms 
	are available for clique-width, little is known about 
	encodings.

	We initiate the quest to understand 
	encoding capabilities with clique-width by considering 
	\emph{abstract argumentation}, which is a robust 
	framework 
	for reasoning with conflicting
	arguments.  It is based on directed graphs and asks for
	computationally challenging properties, making it a natural
	candidate to study computational properties.
	We design novel reductions from argumentation problems to (Q)SAT.
	Our reductions linearly preserve the clique-width, resulting in
	directed decomposition-guided (DDG) reductions.
	We establish novel results for all argumentation semantics,
	including counting.
	Notably, the overhead caused by our DDG reductions cannot be
	significantly improved under reasonable assumptions. 
\end{abstract}

\section{Introduction}
Many problems in combinatorics, symbolic AI, and knowledge
representation and reasoning are computationally very
hard~\cite{EiterG93,Truszczynski11,Dvorak12a}.
%
In the literature, various structural restrictions have been
identified under which problems become
tractable~\cite{Niedermeier06,DowneyFellows13,CyganEtAl15}.
In theory, we are interested whether an efficient algorithm
exists~\cite{Courcelle90a,
	BorieParkerTovey92,FischerMakowskyRavve08,Courcelle18}.
In practice, we seek effective practical
algorithms~\cite{LampisMengelMitsou18,
	JarvisaloLehtonenNiskanen25}.
%
%
From both perspectives, we want to
understand the structural combinatorial core
of the problem. 
Structure preserving reductions to other problems 
support this understanding by its algorithmic and logic-based
definability character.
%
%
%
%
Important structural properties of input instances are for example
hierarchical graph decompositions, which are quite interesting for
algorithmic purposes and for solving problems in polynomial-time in
the input size and exponential in a parameter defined on the
decomposition, e.g., treewidth~\cite{Freuder85,Dechter99}.
%
%
%
%

Even more general than treewidth is the graph parameter
\emph{clique-width}~\cite{CourcelleHRHG}, which can even be small for
dense graphs where treewidth is large, measuring the distance from
co-graphs~\cite{CourcelleOlariu00}.  Dynamic programming algorithms
for deciding acceptance under preferred semantics are known for
clique-width~\cite{DvorakSzeiderWoltran10}, but little is known
about 
structure guided reductions to other problems.
This is particularly
interesting 
when reasoning 
relies on tools such as SAT solving 
where 
research increasingly asks for 
efficient encodings 
and 
theoretical limitations~\cite{HeuleLynceSzeider23} 
%
under the light that logic-based characterizations are known that
allow to solve SAT faster.

%
%
%
%

\begin{proposition}[\cite{FischerMakowskyRavve08}]\label{prop:numsatruntime}
	For Boolean formulas of directed incidence clique-width~$k$ and size~$n$,
	counting SAT can be solved in 
	time~$2^{\BigO({k})}\cdot\poly(n)$. 
\end{proposition}


%
The proposition immediately yields the natural research question:
\emph{Can we encode problems into (Q)SAT while preserving the
	clique-width?}
Indeed, such encodings would allow to easily reuse
Proposition~\ref{prop:numsatruntime} for solving.
%
In this paper, we address 
this question for 
abstract argumentation, a natural knowledge representation
framework widely used for reasoning with conflicting
arguments~\cite{Dung95a,Rahwan07a}.  There, relationships between arguments
are specified in directed graphs, so-called argumentation frameworks
(AFs), and conditions are placed on sets (extensions) of arguments
that allow AFs to be evaluated.
The computational complexity of argumentation is
well-studied~\cite{DvorakSzeiderWoltran10,DvorakPichlerWoltran12,Dvorak12a,CharwatEtAl15,FichteHecherMeier24}.
Since AFs are already given as direct graphs and the
complexity is 
often even beyond NP, it makes 
it a perfect candidate 
to study clique-width aware encoding.

%

\smallskip
\noindent \emph{Our main contributions} are as follows:\\[-1.25em]
\begin{enumerate}
	\item We design reductions from argumentation problems to
	satisfiability of (quantified) Boolean
	formulas. 
	Our reductions are \emph{directed decomposition-guided (DDG)},
	employ k-expressions, and linearly preserve clique-width.
	\item For \emph{all} common \emph{argumentation semantics}, we
	establish favorable upper bounds (tractability) for \emph{extensions
		existence}, \emph{argument acceptance}, and \emph{counting}.
	This also works for the maximization-based (second-level) semantics and demonstrates the flexibility of our approach. For these, we rely on an auxiliary result that establishes how one can convert mixed normal-form QBF matrices
	into DNF and CNF (for inner-most $\forall$ and $\exists$), respectively.
	
	\item We show that the overhead caused by our DDG reductions cannot be
	significantly improved under reasonable assumptions providing
	\emph{structurally optimal reductions}.
	Indeed, this already holds for skeptical reasoning and then
	immediately carries over to the counting problem.
\end{enumerate}

\newcommand{\cons}[1]{\ensuremath{\mathsf{exist}_{#1}}}
\newcommand{\cred}[1]{\ensuremath{\mathsf{c}_{#1}}}
\newcommand{\skept}[1]{\ensuremath{\mathsf{s}_{#1}}}
\newcommand{\cnt}[1]{\#_{#1}}
\begin{table}[t]
	\centering
	\resizebox{.47\textwidth}{!}{
		\begin{tabular}{lcc}
			\toprule
			& \multicolumn{2}{c}{$\skept{\sigma}$($\cred{\sigma}^\star$)/$\cnt{\sigma}$} \\[0.25em]
			$\sigma\in$ &       $\{\stab, \adm, \comp\}$ & $\{\pref, \sigmai, \stag\}$\\
			\midrule
			CW-Aw.   &      $\BigO(k)$                                &   $\BigO(k)$                  \\
			CW-LB (ETH)    & $\Omega(k)$                                     & $\Omega(k)$                    \\
			Rt (UB/LB) & $2^{\theta(k)}\cdot \poly(n)$                                     &    $2^{2^{\theta(k)}}\cdot \poly(n)$                 \\
			Ref. (UB) & Thm~\ref{thm:stab:correct}--\ref{thm:comp:cwaw} & Thm~\ref{thm:pref:correct}--\ref{thm:counting-ub}\\  
			%
			%
			Ref. (LB) & Thm.~\ref{thm:lb}/Corr.~\ref{cor:stabcomp} & Prop.~\ref{thm:slb}\\
			\bottomrule
	\end{tabular}}
	\caption{%
		Overview of our results, where $k = \scw(\mathcal{G}^d_i(F))$,
		and $n = \Card{A}$ for given AF~$F=(A,R)$. 
		%
		%
		%
		We use $\cred{\sigma}$ for credulous acceptance,
		$\skept{\sigma}$ for skeptical acceptance,
		and 
		$\cnt{\sigma}$ is the extension counting problem.
		%
		%
		%
		%
		%
		%
		%
		%
		%
		``CW-Aw.'' refers to the clique-width increase caused by DDG reductions.
		``CW-LB (ETH)'' refers to clique-width lower bounds of DDG reductions under ETH,
		``Rt (UB/LB)'' are runtime upper and ETH lower bounds.
		%
		%
		$^\star$: Results for $\cred\stab$ also apply to $\cons{\stab}$ whereas $\cons{\sigma}$ is trivial for all other semantics. Finally, $\cred{\pref}$ can be solved faster via $\cred{\adm}$ and therefore shares the same bounds.
		%
		%
	}%
	\label{tab:results}
\end{table}


\paragraph{Related Works}
Abstract argumentation is widely studied in 
knowledge representation and
reasoning~\cite{Dung95a,Rahwan07a,AmgoudPrade09a,RagoCocarascuToni18a}.
%
%
%
The computational complexity depends on the considered semantics and
common decision tasks range between polynomial-time and the second
level of the polynomial hierarchy. For example, deciding whether a
given argument belongs to some extension (credulous acceptance) is
\NP-complete for stable semantics and $\SigmaP$-complete for the
semi-stable
semantics~\cite{DunneBench-Capon02a,DvorakWoltran10,Dvorak12a}.
Counting complexity in abstract argumentation is well-studied~
\cite{BaroniDunneGiacomin10,FichteHecherMeier24} with complexity
reaching $\numberDotP$ (admissible, complete, stable) and
$\numberDotCoNP$ (preferred, semi-stable, stage) for extension
counting and $\#\cdot\NP$ (admissible, complete, stable) and
$\#\cdot\SigmaP$ (preferred, semi-stable, stage) for projected
counting.
For treewidth, many results in abstract
argumentation~\cite{DvorakPichlerWoltran12}, including 
and decomposition-guided
reductions~\cite{Hecher20,FichteHecherMahmood21}, are known.
%
%
%
For clique-width, \cite{DvorakSzeiderWoltran10} established dynamic
programming algorithms and tractability results for acceptability
under preferred semantics.
SAT encodings are commonly used for solving argumentation
problems~\cite{NiskanenJarvisalo20} and QBF encodings enable tight
computational
bounds~\cite{LampisMengelMitsou18,FichteHecherPfandler20}.
In propositional satisfiability, tractability results for directed
incidence clique-width~\cite{FischerMakowskyRavve08}, modular
treewidth~\cite{PaulusmaSlivovskySzeider16}, and symmetric incidence
clique-width~\cite{SlivovskySzeider13} and hardness results for
undirected clique-width~\cite{FischerMakowskyRavve08} exist.
Tractability results for validity of QBFs are also known for incidence
treewidth and directed clique-width~\cite{CapelliMengel19}.
%

\section{Preliminaries}
%
We assume that the reader is familiar with standard terminology in
Boolean logic~\cite{DBLP:series/faia/336}, computational complexity~\cite{Papadimitriou94}, and parameterized
complexity~\cite{CyganEtAl15}.
%
%
For an integer~$k$, let $[k]\eqdef \{1,\ldots, k\}$ and $A \dcup B$
be the union over disjoint sets~$A,B$.
%
%
%
%


\paragraph{Satisfiability}
A literal is a (Boolean) variable~$x$ or its negation~$\neg x$.
A \emph{clause} or \emph{cube (also known as term)} is a finite set of literals, interpreted
as the disjunction or conjunction of these literals, respectively.
%
%
%
A \emph{CNF formula} or \emph{DNF formula} is a finite set of clauses,
interpreted as the conjunction or disjunction of its clauses or cubes.
Sometimes we say formula to refer to a CNF formula.
We use the usual convention that an empty conjunction 
corresponds to $\top$ and an empty disjunction 
to~$\bot$.
%
%
%
Let $\varphi$ be a formula.
%
%
For a clause~$c \in \varphi$, we let $\var(c)$ consist of all variables that
occur in~$c$ and $\textstyle\var(\varphi)\eqdef\bigcup_{c \in \varphi} \var(c)$.
An \emph{assignment} is a mapping
$\alpha: \var(\varphi) \rightarrow \{0,1\}$, $\alpha$ is total if it maps
all variables in~$\varphi$.
For $x\in \var(\varphi),$ we define $\alpha(\neg x) \eqdef 1 - \alpha(x)$.
The CNF formula~$\varphi$ \emph{under the
	assignment~$\alpha \in \ta{\var(\varphi)}$} is the formula~$\varphi_{|\alpha}$
obtained from~$\varphi$ by removing all clauses~$c$ containing a literal set
to~$1$ by $\alpha$ and removing from the remaining clauses all
literals set to~$0$ by $\alpha$. An assignment~$\alpha$ is
\emph{satisfying} if $\varphi_{|\alpha}=\emptyset$ and $\varphi$ is
\emph{satisfiable} if there is a satisfying assignment~$\alpha$. %
%
%
%
%
%
%
%
\SAT asks to decide satisfiability of~$\varphi$ and 
\cSAT asks for its number of total satisfying assignments.
%
%
%

\paragraph{Quantified Boolean Formulas}
Let $\ell$ be a positive integer, which we call \emph{(quantifier) rank} later, and $\top$ and $\bot$ be the constant always evaluating to $1$ and $0$, respectively.
%
%
A \emph{quantified Boolean formula}~$\phi$ (in prenex normal form), \emph{qBf} for short, is an expression of the form $\phi=Q_1 X_1.Q_2 X_2.\cdots Q_\ell X_\ell. F(X_1,\dots,X_\ell)$, where for $1\leq i\leq \ell$, we have $Q_i\in\{\forall,\exists\}$ and $Q_i \neq Q_{i+1}$, the $X_i$ are disjoint, non-empty sets of Boolean variables, and $F$ is a Boolean formula. We let $\matr(\phi)\dfn F$.
%
We evaluate $\phi$ by $\exists x.\phi\equiv \phi[{x\mapsto 1}]\lor\phi[{x\mapsto 0}]$ and $\forall x.\phi\equiv \phi[{x\mapsto 1}]\land\phi[{x\mapsto 0}]$ for a variable~$x$.
%
%
%
%
%
%
%
%
W.l.o.g.\ we assume that $\matr(\phi)=\psi_{\text{CNF}} \wedge \psi_{\text{DNF}}$,
where $\psi_{\text{CNF}}$ is in CNF 
and $\psi_{\text{DNF}}$ is in DNF. 
The evaluation of DNF formulas is defined in the usual way.
%
%
Then, depending on $Q_\ell$, either $\psi_{\text{CNF}}$ or $\psi_{\text{CNF}}$ is optional, more precisely,  $\psi_{\text{CNF}}$ might be $\top$, if $Q_\ell=\forall$, and $\psi_{\text{DNF}}$ is allowed to be $\top$, otherwise.
The problem $\ell\hy\QSAT$ asks, given a closed qBf $\phi=\exists X_1.\phi'$ of
rank~$\ell$, whether $\phi\equiv1$ holds.
%
The problem $\#\ell\hy\QSAT$ asks, given a closed qBf $\exists X_1. \phi$ of rank~$\ell$, to count assignments~$\alpha$ to~$X_1$ such that $\phi[\alpha]\equiv1$.
For brevity, we sometimes omit~$\ell$.


\paragraph{Computational Complexity}
For integer $i\geq 0$, $\exp(i,p)$ means a tower of exponentials, i.e.,
$\exp(i - 1,2^p)$ if $i > 0$ and $p$ if $i=0$.
We assume that $\poly(n)$ is any polynomial for given positive integer
$n$.
The Exponential Time Hypothesis (ETH)~\cite{ImpagliazzoPaturiZane01}
is a widely accepted standard hypothesis in the fields of exact and
parameterized algorithms. 
ETH states that there is some real~$s > 0$ such that we cannot decide
satisfiability of a given 3\hy CNF formula~$\varphi$ in
time~$2^{s\cdot\Card{\varphi}}\cdot\CCard{\varphi}^{\mathcal{O}(1)}$~\cite[Ch.14]{CyganEtAl15},
where $\Card{\varphi}$ refers to the number of variables and $\CCard{\varphi}$ to
the \emph{size} of~$\varphi$, which is number of variables and 
clauses in~$\varphi$.

\paragraph{Abstract Argumentation}
We use the argumentation terminology by Dung~\cite{Dung95a} and
consider non-empty, finite sets~$A$ of arguments.
An \emph{argumentation framework~(AF)}
is a directed graph~$F=(A, R)$ where $A$ is a
set of elements, called \emph{arguments}, and $R \subseteq A\times A$,
a set of pairs of arguments representing direct attacks of arguments.
%
%
%
%
An argument~$s \in A$ is called \emph{defended by $S$} 
if for every $(s', s) \in R$, there exists $s'' \in S$ such that
$(s'', s') \in R$.  The family~$\adef_F(S)$ is defined by
$\adef_F(S) \eqdef\{ s \mid s \in A, s \text{ is defended by } S
\text{ in } F \}$.
In argumentation, we are interested in computing so-called
\emph{extensions}, which are subsets~$S \subseteq A$ of the arguments
that meet certain properties according to certain semantics.
%
%
We say $S \subseteq A$ is \emph{conflict-free} 
if $(S\times S) \cap R = \emptyset$; $S$ is
\emph{admissible} 
if (i) $S$ is \emph{conflict-free}, 
and (ii) every $s \in S$ is \emph{defended by~$S$}. 
Assume an admissible set~$S$.  Then, (iiia) $S$ is \emph{complete}
if $\adef_F(S) = S$; (iiib) $S$ is~\emph{preferred}, 
if there is no $S' \supset S$ that is \emph{admissible}; 
(iiic) $S$ is \emph{semi-stable} 
if there is no admissible set $S' \subseteq A$ 
with~$S^+_R\subsetneq (S')^+_R$ where
$S^+_R:=S\cup\SB a\SM (b,a)\in R, b \in S \SE$;
%
%
%
%
%
%
(iiid) $S$ is \emph{stable} 
if 
every $s \in A \setminus S$ is \emph{attacked} by some $s' \in S$.  A
conflict-free set~$S$ is \emph{stage} 
if there is no conflict-free set~$S'\subseteq A$ 
with~$S^+_R\subsetneq (S')^+_R$.
%
%
%
%
%
We denote semantics by acronyms $\adm, \comp, \pref, \sigmai, \stab,$
and $\stag$, respectively.  For a
semantics~$\sigma \in \{\adm, \comp, \pref, \sigmai, \stab, \stag\}$,
$\sigma(F)$ is the set of \emph{all extensions} of
semantics~$\sigma$ in~$F$.
%
%
%
%
%
%
Given an AF~$F=(A,R)$. Problem~$\cons{\sigma}$ asks
whether~$\sigma(F) \neq \emptyset$;
$\cnt{\sigma}$ asks for $\Card{\sigma(F)}$; 
additionally given argument $a \in A$, \emph{credulous
acceptance} $\cred{\sigma}$ asks whether
$a \in \bigcup_{e \in \sigma(F)}e$; and
\emph{skeptical acceptance} $\skept{\sigma}$ asks whether
$a \in \bigcap_{e \in \sigma(F)}e$.

\begin{figure}[t]
\begin{center}
\resizebox{0.2\textwidth}{!}{
\begin{tikzpicture}[
baseline=(current bounding box.center),
every node/.style={rounded corners, draw, inner ysep = 3pt},
execute at end node={\strut}]
%
\node (a) at (1,0) [] {\textbf{z}oo};
\node (b) at (3,1.5) {\textbf{o}utback};
\node (c) at (5,0) {\textbf{u}npredictable};
\node (d) at (5,1.5) {\textbf{r}anger};
\foreach \f/\t in { a/b, c/b, c/d, d/c}{
\draw [-stealth', thick] (\f) to (\t);
}
\end{tikzpicture}
}%
\end{center}
\caption{An example framework for deciding between
going to see kangaroos in a zoo or in the outback.
}
\label{fig:runningex}
\end{figure}
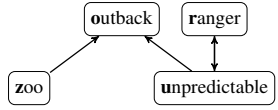

\begin{example}\label{ex:running}
Consider an AF $F$ with 4 arguments as depicted in
Figure~\ref{fig:runningex} arguing about watching kangaroos.
Watching kangaroos in a ``zoo'' gives you all the excitement without
needing to go the ``outback''. 
To observe them naturally, you need to see them in the ``outback''.
%
%
However, kangaroos are ``unpredictable'' and can be dangerous.
%
Regardless, you can go for a tour with a ``ranger'' making it
safe to observe kangaroos in the wild, so it is not that dangerous.
\end{example}

\paragraph{Incidence Graphs and Clique-width}
%
%
We follow standard terminology for graphs  and directed (multi) graphs~\cite{BondyMurty08}.
%
%
The \emph{directed incidence graph}~$\mathcal{G}^d_i(\varphi)$ of CNF (DNF) formula~$\varphi$~\cite{OrdyniakPS13}  
is a
bipartite graph with the variables and clauses (terms) of~$\varphi$ as vertices and
a 
directed edge between those, indicating whether variables occur
positively or negatively in a clause (term), respectively.
%
%
%
%
%
The \emph{incidence graph}~$\mathcal{G}_i(\varphi)$ of~$\varphi$ omits 
edge directions from~$\mathcal{G}^d_i(\varphi)$.
%
%
%
%
%
%
%
%
%
%
We use standard definitions for
clique-width~\cite{CourcelleHRHG,CourcelleMSOLAHO}.
%
%
%
%
%
%
%
Intuitively, treewidth~\cite{Bodlaender06} measures the distance of a
graph from being a tree and clique-width 
measures the distance of a graph to a co-graph.
Clique-width is bounded by treewidth, see~\cite{CorneilOTRBCWAT}.
A graph 
is a \emph{co-graph}, if it can be constructed as
follows:
(i)~a graph with one vertex is a co-graph;
for two co-graphs~$G_1=(V_1,E_1)$ and $G_2=(V_2,E_2)$, (iia)~the
disjoint union $G_1 \oplus G_2 \eqdef (V_1 \dcup V_2, E_1 \dcup E_2)$
is a co-graph; and
%
%
(iib)~the disjoint sum
$G_1 {\times} G_2 \eqdef (V_1 \dcup V_2,\allowbreak E_1 \dcup E_2 \dcup \SB \{u,v\}
\SM u \in V_1, v\in V_2 \SE)$ is a co-graph.
%
%
%

This lifts to $k$-graphs where $k$ is a positive integer
with $k$ labels, called \emph{colors}.
The labeling of a graph $G = (V, E)$~is a
function~$\lambda: V \rightarrow [k]$ and a \emph{$k$-graph} is a graph whose vertices are labeled by integers
from~$[k]$. 
%
%
An \emph{initial $k$-graph} consists of exactly one vertex~$v$
colored by $c \in [k]$, denoted by $c(v)$, e.g., $1(v)$ is a shorthand
for $G$ that is a vertex~$v$ of color~$1$.
Now, we can construct a graph~$G$ from initial $k$-graphs by
repeatedly applying the following three \emph{operations}.
(i)~\emph{Disjoint union}, denoted by $\oplus$;
(ii)~\emph{Relabeling}: changing all colors~$c$ to $c'$, denoted by
$\rho_{c\rightarrow c'}$;
(iii)~\emph{Edge introduce}: 
connecting all vertices colored by $c$ with all vertices colored by
$c'$, denoted by $\eta_{c,c'}$ or $\eta_{c',c}$; already existing edges
are not doubled.
A construction of a $k$-graph~$G$ using these operations can be
represented by a \emph{$k$-expression}, which is an algebraic term composed of $c(v)$, $\oplus$,
$\rho_{c\rightarrow c'}$, and $\eta_{c,c'}$ where $c, c' \in [k]$ and
$v$ is a vertex. 
To construct directed graphs, $\eta_{c,c'}$ is interpreted as the operation to introduce directed edges from vertices labeled $c$ to those  labeled $c'$.
%
%
%
%
%

We describe a $k$-expression by a parse-tree~$T=(V_T, E_T)$ and use
\emph{parse-tree of width~$k$} synonymously.
We refer by $\children(b)$ to the \emph{set of children} of a
node~$b$ in $V_T$.
We define $\colors: V_T \rightarrow 2^{[k]}$ that yields
the \emph{set of colors} in the graph $G=(V,E)$ constructed \emph{up to operation~$b$}.
Additionally, $\col: (V \times V_T) \rightarrow {[k]}$ 
gives the \emph{color of a vertex in $G$ up to operation $b$}.
The \emph{last operation} is $rt$ for \emph{root}.
%
%
%

The \emph{clique-width~$\cw{G}$} of an \emph{undirected graph} is the smallest~$k$ such
that $G$ is definable by a $k$-expression.
%
%
%
%
The \emph{directed clique-width~$\scw(G)$} of a \emph{directed graph} is the smallest~$k$ such that $G$ is definable by a $k$-expression.
The \emph{directed incidence clique-width} of a formula $\varphi$ is~$\scw(\mathcal{G}^d_i(\varphi))$.

\paragraph{Graphs and Clique-width}
Let $G=(V,E)$ be a (di)graph.
A graph~$G'=(V',E')$ is a sub-graph of~$G$ if $V'\subseteq V$.
For a set $V'\subseteq V$, the \emph{induced
	sub-graph}~$G[V']=(V',E')$ consists of vertices in~$V'$ and all
edges from the original graph on those vertices, i.e.,
~$E' \eqdef \SB (v_1,v_2) \SM (v_1,v_2) \in E, v_1,v_2 \in V'\SE$ of
edges.
An \emph{independent set} (also called stable set or co-clique) of~$G$
is a set~$I \subseteq V$ in which no two vertices are adjacent. A set
is independent if and only if it is a clique in the graph's
complement. An independent set is a \emph{maximal independent set
	(MIS)}, if there is no superset that is an independent set.
A \emph{co-graph} can be seen as a graph where any maximal clique in
any of its induced subgraphs intersects any maximal independent set in
that subgraph in exactly one vertex.

\paragraph{Directed Incidence Clique-width for CNF/DNF}
Let $\varphi = \{c_1,\dots,c_m\}$ be a formula over variables $X$. The directed incidence graph $\mathcal{G}^d_i(\varphi)$ of $\varphi$ is a bipartite graph with vertex set $X\cup \varphi$ and edge set $ \{ (c,x) | c\in \varphi \text{ and } x \in c\} \cup \{ (x,c) | c\in \varphi \text{ and } \neg x \in c\} $. 
The (undirected) incidence graph $\mathcal{G}_i(\varphi)$ of $\varphi$  omits the edge direction from  $\mathcal{G}^d_i(\varphi)$.
Observe that the orientation of edges can also be encoded by labeling them with the signs $\{+,-\}$ such that an edge between a variable $x$ and a clause $c$ is labeled $+$ if $x \in c$ and  $-$ if $\neg x \in  C$. 
This results in the \emph{signed incidence graph}~\cite{FischerMakowskyRavve08} that encodes the same information as the directed incidence graph~\cite{OrdyniakPS13}.
The directed incidence graph for DNF is defined similarly after replacing clauses  by terms.

\newcommand{\bag}[1]{#1}
\definecolor{myone}{HTML}{f58231}
\definecolor{mytwo}{HTML}{4363d8}
\definecolor{mythree}{HTML}{00e600}

\newcommand{\cu}[1]{\textcolor{myone}{#1}}
\newcommand{\cd}[1]{\textcolor{mytwo}{#1}}
\newcommand{\ct}[1]{\textcolor{mythree}{#1}}

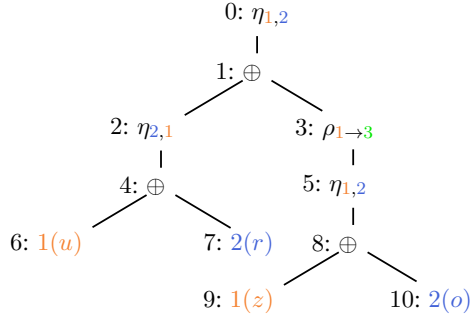
\begin{figure}[t]
\vspace{-1em}
\begin{center}
\resizebox{0.33\textwidth}{!}{
\begin{tikzpicture}[
every node/.style = {font=\small},
level distance=0.9cm,
sibling distance=3cm,
edge from parent/.style={draw,thick},
every path/.style={thick},
label/.style={draw=none},
circ/.style={circle,draw,inner sep=2pt},
align = center
]

\node {$\bag{0}$: $\eta_{\cu{1},\cd{2}}$}
child {node {\hspace{-1.5em}$\bag{1}$: $\oplus$}
child {node {\hspace{-1.5em}$\bag{2}$: $\eta_{\cd{2},\cu{1}}$}
child {node {\hspace{-1.5em}$\bag{4}$: $\oplus$}
child {node {\hspace{-1.5em}$\bag{6}$: $\cu{1(u)}$}}
child {node {\hspace{-1.5em}$\bag{7}$: $\cd{2(r)}$}}}
}
child {node[label] {\hspace{-1.5em}$\bag{3}$: $\rho_{\cu{1}\rightarrow \ct{3}}$}
child {node {\hspace{-1.5em}$\bag{5}$: $\eta_{\cu{1},\cd{2}}$}
child {node {\hspace{-1.5em}$\bag{8}$: $\oplus$}
child {node {\hspace{-1.5em}$\bag{9}$: $\cu{1(z)}$}}
child {node {\hspace{-1.5em}$\bag{10}$: $\cd{2(o)}$}}
}}
}
};

\end{tikzpicture}
}
\end{center}
\caption{A $k$-expression of $F$ (directed graph)
from Example~\ref{ex:running}, illustrated as parse tree.
We color the labels in the figure to distinguish between operation numbers and colors. 
%
}
%
\label{fig:excw}
\end{figure}

\begin{example}[Cont.]\label{ex:argcw}
Consider our argumentation framework~$F$ from
Example~\ref{ex:running}.
The clique-width of~$F$ is $\leq 3$, since only three colors
are needed to draw this AF. We illustrate this in
Figure~\ref{fig:excw} where 
concrete operations are labeled by numbers $\{0,\dots,10\}$ for ease and $0$ refers to the root. 
%
\end{example}

\begin{proposition}[\cite{FischerMakowskyRavve08}]\label{prop:cwbounds}
For any Boolean formula $\varphi$,
$\cw{\mathcal{G}_i(\varphi)}\leq 2 \cdot \scw(\mathcal{G}^d_i(\varphi))$.
\end{proposition}

There is also a QSAT version of Proposition~\ref{prop:numsatruntime}.

\begin{proposition}[cf.~\cite{CapelliMengel19}]\label{prop:mengel}
For QBFs of directed incidence clique-width~$k$, quantifier
depth~$\ell$, and size~$n$,
counting QSAT (on free variables) can be solved in
time~$\tower({\ell+1},\BigO(k))\cdot\poly(n))$. 
\end{proposition}

\noindent For fixed $k$ and an undirected graph $G$, 
one can find an $f(k)$-expression 
of clique-width $k$ in polynomial time~\cite{OumS06}; similarly for 
directed
graphs~\cite{Kante07}.

\section{K-Expression-Aware Encodings}

We proceed towards reducing from argumentation problems to
satisfiabilty of (quantified) Boolean formulas by a \emph{directed
	decomposition-guided (DDG)} reduction that employs $k$-expressions.
To this aim, let $F=(A,R)$ be an abstract argumentation framework from
which we construct a QBF.
There, we use a variable~$e_a$ for every argument $a\in A$ to store
whether $a$ is in the extension or not.
We also use auxiliary variables to cover the state of attacks.
In our encoding, we employ the structure and need to take care in
particular of the following:
(i)~\emph{initial color of arguments}, which is where we can decide whether
the (single) color is defeated by being in the extension;
(ii)~\emph{merge of colors}, which might occur during disjoint union
or relabeling; and
(iii)~\emph{edge introduce}, which occurs when drawing edges from
color~$c'$ to $c$.


\subsection{Basic Semantics}\label{sec:first}

We start by presenting reductions for semantics whose classical
complexity is located on the first level of the polynomial
hierarchy,~i.e., stable, admissible, complete extensions.
%
	For space reasons, we 
	defer proof details 
	to a self-archived extended version. 
	%
	%
	In the following, we assume that we have a $k$-expression that defines
	the considered argumentation framework (directed graph), while
	requiring only $k$ different colors.
	We guide our encoding along the $k$-expression. 
	%
	%
	%

\paragraph{\underline{Stable Extensions}}
Here, we require conflict-freeness and that all arguments, which are not in
the extension, are attacked by the extension.  We encode this via a
set $E$ of \emph{extension variables}.
In the initial operation~$c(a)$ for argument~$a$ of color $c$,
the variable~$e_a$ encodes whether argument~$a$ is in the extension.
Then, variable~$e^b_c$ encodes whether $c$ contains an extension variable in
operation~$b$, which we refer to by \emph{extension color}.
%
%
	To use the fact that colors have extension members, we guide the
	information along the $k$-expression. 
	Finally, we ensure conflict-freeness of extension arguments in the
	(directed) edge introduction operations. 
	Now,
	we construct our encoding formally.
	%
	
	\begin{flalign}
		\label{stab:leaf}&{e}_{c}^b \leftrightarrow e_a	&& \hspace{-4em}\text{initial $b$, create $a$ of color $c$}\\
		\label{stab:union}&{e}_{c}^b \leftrightarrow \hspace{-2em}\bigvee_{\substack{b'\in\children(b): c\in \colors(b')}}\hspace{-3em} {e}_{c}^{b'}	&& \hspace{-4em}\text{disjoint union $b$, every $c\in \colors(b)$}\\
		&{e}_{c}^b \leftrightarrow \hspace{-1.75em}\bigvee_{\text{if $b$ relabeling\ }c'\mapsto c}\hspace{-2.25em} {e}_{c'}^{b'} \vee \hspace{.5em}\bigvee_{\text{if\ }c\in \colors(b')}\hspace{-1.5em}e_c^{b'} 	&& \hspace{-0em}\text{relabeling $b$, every $c{\,\in\,} \colors(b),$}\raisetag{1.35em}\notag\\[-1.6em]
		&&&b'\in\children(b)\label{stab:relabel}\\
		&{e}_{c}^b \leftrightarrow {e}_{c}^{b'}, \quad \hspace{-2.5em}\bigvee_{\substack{\text{if $b$ edge introduce } (c',c)}} \hspace{-3.25em}\neg e_c^b \vee\neg e_{c'}^b 	&& \text{edge intr.\ $b$, every $c{\,\in\,} \colors(b),$}\raisetag{1.25em}\notag\\[-1.6em]
		&&& b'\in\children(b)\label{stab:edge}
\end{flalign}

\postrebuttal{Note that the child $b'$ in Eq.~(\ref{stab:relabel}) is unique. However, $c$ could be a fresh color only introduced in $b$ (so $b'$ does not consider $c$). 
In this case, the disjunction would be empty (a case that is also covered here).}
Here, equations~(\ref{stab:leaf})--(\ref{stab:edge}) suffice to model conflict-freeness,
yielding formula $\varphi_{\#\conf}=\mathcal{R}_{\CConf\rightarrow\CBSAT}(F,\ptlong)$.
As stable extensions also attack non-extension arguments,
we require for every color $c$ that is used 
in an operation~$b$ an auxiliary variable $d_{c}^b$ that indicates whether 
\emph{every non-extension argument} of color $c$ up to operation~$b$ is \emph{attacked by the extension}. 
This is achieved via a set~$D$ of \emph{defeated} variables,
guided along the $k$-expression. 

\begin{flalign}
	\label{stab:defeat-leaf}&d_{c}^b \leftrightarrow e_c^b	&& \hspace{-4em}\text{initial $b$, every $c\in \colors(b)$}\\
	\label{stab:defeat-union}&d_{c}^b \leftrightarrow \hspace{-2em}\bigwedge_{\substack{b'\in\children(b): c\in \colors(b')}} \hspace{-3em} d_{c}^{b'}	&& \hspace{-4em}\text{disjoint union $b$, every $c\in \colors(b)$}\\
	&d_c^b \leftrightarrow \hspace{-2.5em}\bigwedge_{\text{if $b$ relabeling\ }c'\mapsto c} \hspace{-2em}d_{c'}^{b'} \wedge \bigwedge_{\text{if\ }c\in \colors(b')}\hspace{-1.65em} d_{c}^{b'}
	&& \hspace{-.0em}\text{relabeling $b$, every $c\in \colors(b)$,}\notag\\[-1.6em]
	&&&b'\in\children(b)\label{stab:defeat-relabel}\\
	&d_{c}^b \leftrightarrow d_{c}^{b'} \vee \hspace{-3em}\bigvee_{\text{if $b$ edge introduce\ }(c',c)} \hspace{-2.75em}e_{c'}^b 	&& \hspace{-1.8em}\text{edge introduce $b$, every $c\in \colors(b)$,}\notag\\[-1.6em] 
&&&\hspace{-1.8em}b'\in\children(b)\label{stab:defeat-edge}
\end{flalign}

\futuresketch{
To mark color sets that have to be defeated, which are those containing at least one argument not in the extension, we guide this information along the expressions.}

\noindent In the root $rt$, we ensure that all colors are defeated. 
\begin{flalign}
\label{stab:defeat-root}{d}_c^{rt} 	&& \text{for root colors\ } c\in \colors(rt)
\end{flalign}


\paragraph{Intuition}
In the initial $k$-graphs $b=c(a)$, we remember whether $c$ is an extension
color, this information is 
guided along the $k$-expression via other formulas.  The
\emph{disjunctive} encoding, in
Formulas~(\ref{stab:leaf})--(\ref{stab:edge}), guarantees that $c$ is
an extension color if it is an extension color in some operation $b$. Then,
Formula~(\ref{stab:leaf}) allows to retrieve arguments from extension
colors and Formula~(\ref{stab:edge}) requires conflict-freeness of
those arguments.  Moreover, given an extension color, the initial graphs
uniquely determine arguments in an extension.  Finally,
Formulas~(\ref{stab:defeat-leaf})--(\ref{stab:defeat-edge}) encode
whether each argument is either in the extension, or attacked by
it. This is again encoded via colors, where a \emph{conjunctive} encoding is
used to enforce that each argument of a color has been considered.

\newcommand{\cvu}[1]{\textcolor{myone}{#1}}
\newcommand{\cvd}[1]{\textcolor{mytwo}{#1}}
\newcommand{\cvt}[1]{\textcolor{mythree}{#1}}

\begin{example}[cont.] \label{ex:stable} Consider $F$ from
Example~\ref{ex:running} and the operations corresponding to the
parse tree from Figure~\ref{fig:excw}. 
%
%
%
%
%
%
For operation~$\bag{9}$, which creates the initial graph~$\cu{1(z)}$, we add
${e}_{{\cvu{1}}}^{\bag{9}} \leftrightarrow e_z$ by
Formula~(\ref{stab:leaf}) 
and ${d}_{{\cvu{1}}}^{\bag{9}} \leftrightarrow {e}_{{\cvu{1}}}^{\bag{9}}$ by
Formula~(\ref{stab:defeat-leaf}).
%
%
%
%
%
%
%
For operation~$\bag{8}$ \emph{(disjoint union)}, 
we add the formulas
${e}_{{\cvu{1}}}^{\bag{8}} \leftrightarrow {e}_{{\cvu{1}}}^{\bag{9}}$,
${e}_{{\cvd{2}}}^{\bag{8}} \leftrightarrow {e}_{{\cvd{2}}}^{\bag{10}}$ by
Formulas~(\ref{stab:union}) and
${d}_{{\cvu{1}}}^{\bag{8}} \leftrightarrow {d}_{{\cvu{1}}}^{\bag{9}}$,
${d}_{{\cvd{2}}}^{\bag{8}} \leftrightarrow {d}_{{\cvd{2}}}^{\bag{10}}$
by Formulas~(\ref{stab:defeat-union}).
%
%
%
%
For operation~$\bag{5}$ \emph{(edge-introduce)}, 
we obtain ${e}_{\cvu{1}}^{\bag{5}} \leftrightarrow {e}_{\cvu{1}}^{\bag{8}}$,
${e}_{\cvd{2}}^{\bag{5}} \leftrightarrow {e}_{\cvd{2}}^{\bag{8}}$, and
$\neg e_{\cvd{2}}^{\bag{5}} \vee\neg e_{\cvu{1}}^{\bag{5}}$ by
Formulas~(\ref{stab:edge})
and
${d}_{{\cvu{1}}}^{\bag{8}} \leftrightarrow {d}_{{\cvu{1}}}^{\bag{9}}$,
${d}_{{\cvd{2}}}^{\bag{8}} \leftrightarrow {d}_{{\cvd{2}}}^{\bag{10}}$ by
Formulas~(\ref{stab:defeat-edge}).
%
%
%
%
%
For operation~$\bag{3}$ \emph{(relabeling)},
we obtain, ${e}_{\cvd{2}}^{\bag{3}} \leftrightarrow {e}_{\cvd{2}}^{\bag{5}}$,
${e}_{\cvt{3}}^{\bag{3}} \leftrightarrow {e}_{\cvu{1}}^{\bag{5}}$ by
Formulas~(\ref{stab:relabel})
and
${d}_{\cvd{2}}^{\bag{3}} \leftrightarrow {e}_{\cvd{2}}^{\bag{5}}$,
${d}_{\cvt{3}}^{\bag{3}} \leftrightarrow {e}_{\cvu{1}}^{\bag{5}}$ by
Formulas~(\ref{stab:defeat-relabel}).
Rest formulas are analogous.
\end{example}


From now on, we denote our reductions by \emph{directed decomposing
guided (DDG) reduction~$\mathcal{R}_{\#\sigma\rightarrow\CBSAT}$} and
use this notion also for other semantics~$\sigma$.  Moreover, we
denote the resulting (\#)SAT-instance by \emph{$\varphi_{\#\sigma}$}.
%
%
%
First, we prove the correctness of our reduction for stable semantics.
\begin{restatable}[$\star$
,Correctness]{theorem}{stabcorrect}\label{thm:stab:correct}
Let~$F=(A,R)$ be an AF and~$\ptlong$ be a $k$-expression of~$F$. 
Then, the 
DDG reduction~$\mathcal{R}_{\CStab\rightarrow\CBSAT}$ is correct, that is, $\CStab$ on $F$ coincides with $\CBSAT$ on $\mathcal{R}_{\CStab\rightarrow\CBSAT}(F,\ptshort)$.
\end{restatable}

\begin{proof}[Proof (Sketch)]
We establish a bijective correspondence between stable extensions of $F$  and satisfying assignments of $\varphi_\CStab$.
In the forward direction: we construct a unique satisfying assignment $\alpha$ from a given stable extension $S$ of $F$.
This is achieved by setting the truth values for extension variables $E$ according to $S$, i.e., $\alpha(e_a)=1$ iff $a\in S$.
Moreover, we set the value of $\alpha(e_c^b)$  according to the color $c$ and operation $b$ in the $k$-expression to simulate the propagation of the evaluation along the parse-tree.
Then we repeat the same construction for defeated variables $D$.
Our construction of $\alpha$ from $S$ ensures that $\alpha\models\varphi_\CStab$.
To prove the uniqueness of the assignment: we observe that $S$ uniquely determines the evaluation of $\alpha$ for variables $e_a\in E$, whereas the value of $\alpha$ for remaining variables is simply propagated due to the nature of formulas in $\varphi_\CStab$.

In the reverse direction: we construct a unique stable extension $S$ by considering a satisfying assignment $\alpha$ for  $\varphi_\CStab$.
Once again, $S$ is obtained via the extension variables in $ E$ by letting $S$ contain an argument $a$ iff $\alpha(e_a)=1$. Then, it follows from Formulas~\ref{stab:edge} and \ref{stab:defeat-root} that $S$ is stable.
\end{proof}

Next, we establish that our encodings linearly preserve the
clique-width.  That is, reducing an AF instance to a $\BSAT$-instance
increases the clique-width only linearly.
\postrebuttal{Moreover, the size  of the obtained formula grows with the size of the input decomposition and a cw-expression for the SAT instance does not have to be computed from scratch, given a cw-expression for the AF.}

\begin{restatable}[$\star$,CW-Awareness]{theorem}{stabcwaw}\label{thm:stab:cwaw}
Let~$F$ 
be an AF and~$\ptlong$ be a $k$-expression 
of~$F$. 
%
The DDG reduction $\mathcal{R}_{\CStab\rightarrow\CBSAT}(F,\ptshort)$
constructs a SAT instance~$\psi$ that linearly preserves the
directed clique-width,~i.e., $\scw(\sinc{\psi})\in\mathcal{O}(k)$.
\end{restatable}

\begin{proof}[Proof (Sketch)]
Given an AF~$F$ 
and a $k$-expression~$\ptlong$ 
of~$F$, we construct a $k$-expression~$\ptlong'$ 
of~$\sinc{\psi}$ as follows.
We first specify additional colors.
For each color~$c$, we need an
\emph{extension-version}~$e_c$ and a
\emph{defeat-version}~$d_c$ of $c$.
%
%
Then, we additionally need two more copies for each of these versions, called
\emph{child-versions} $ec_c$ and
$dc_c$ to 
add clauses corresponding to the child-operation
$b'$ of an operation $b$ and \emph{expired-versions}
$ex_c$ and
$dx_c$ to simulate the effect that all the edges between certain
variables and their clauses have been added. 
This is required, since otherwise, we will keep adding edges from the child-versions in future, which is undesirable.
Finally, to handle both positive and negative literals in clauses, we duplicate current and the child versions of each colors for $x\in \{e_c,ec_c,d_c,dc_c\}$, denoted as $x^+$ for positive and $x^-$ for negative literals. 
%
For clauses, we use two additional colors, called \emph{clause-making}
($cm$) and \emph{clause-ready} ($cr$) to add 
edges between clauses
and their respective literals, since we are in the setting of incidence graphs.
Intuitively, when adding 
edges between clauses and their respective
variables, we initiate a clause $C$ to be of color $cm$ and its literals $x$ to be of the appropriate color $x^+$ or $x^-$.
Then, we draw directed edges between $C$ and its literals, and later 
relabel $C$ to $cr$ to
avoid any further edges to/from $C$.
Similarly, when going from an operation~$b'$ to its parent $b$, we
change the labels of extension and defeated variables in $b'$ to
their child-versions.
We conclude by observing that the above mentioned additional colors suffice to construct a $k$-expression $\ptlong'$   of~$\sinc{\psi}$, leading to a requirement of $11k + 2$ colors.
\end{proof}

Observe that one might be tempted to consider AFs as undirected graphs (e.g., for the stable semantics).
However, this would not work as we outline below.
\postrebuttal{Two AFs could be ``{same}'' considering undirected edges (or, say symmetric edges) and can thus be constructed using the same labels.
However, one cannot insert directed edges using the labels for undirected edges.
Thus, information is lost when encoding only undirected edges in an AF, even for stable semantics.
Importantly, one can construct AFs where the directed and undirected clique-width differs from each other and such that both AFs admit different number of extensions.}
We next illustrate an example to further highlight this aspect. 


\begin{example}\label{ex:directed}
Consider two AFs $F_1, F_2$ with arguments $\{a,b,c\}$ and $ \{d,e,f\}$, respectively, together with attacks between them as depicted in Figure~\ref{fig:directedex}.
Now, considering undirected edges (or, say symmetric edges) the two AFs are the ``{same}''.
Thus, both can be constructed using the same labels. 
However, one cannot insert directed edges using the same labels.
Observe that $F_1$ has a directed clique-width of two, whereas $F_2$ (directed cycle) has three. 
Importantly, $F_1$ admits three stable extension $\{\{a\}, \{b\}, \{c\}\}$; $F_2$ has none.

Our construction relies on directed clique-width.
Hence, the resulting formulas (in particular, due to Eqs.~(\ref{stab:edge}) and~(\ref{stab:defeat-edge})) corresponding to the two AFs are also different.
%
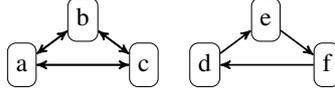
\begin{figure}[t]
\begin{center}
\resizebox{.25\textwidth}{!}{
\begin{tikzpicture}[
	x=1cm,y=0.5cm,
	baseline=(current bounding box.center),
	every node/.style={rounded corners, draw, inner ysep = 3pt},
	execute at end node={\strut}]
	\node (a) at (2,0) [] {a};
	\node (b) at (3,1.5) {b};
	\node (c) at (4,0) {c};
	\foreach \f/\t in { a/b, b/a, b/c, c/b, a/c,c/a}{
		\draw [-stealth', thick] (\f) to (\t);
	}
	
	\node (d) at (5,0) [] {d};
	\node (e) at (6,1.5) {e};
	\node (f) at (7,0) {f};
	\foreach \f/\t in { d/e, e/f, f/d}{
		\draw [-stealth', thick] (\f) to (\t);
		
	}
	
\end{tikzpicture}
}%
\end{center}
\caption{Example AFs highlighting the scenario where edge directions (hence resulting clique-width) matters.}
\label{fig:directedex}
\end{figure}
\end{example}

\paragraph{\underline{Admissible Extensions}}
For admissible extensions, we require conflict-freeness and that every
attack from outside is defeated.
Therefore, we reuse Equations~(\ref{stab:leaf})--(\ref{stab:edge}) to
model conflict-free extensions.  Furthermore, we need to maintain attacks
between colors containing extension arguments.  This is achieved via
\emph{attack variables} $A$, specified as follows.

\begin{flalign}
\label{adm:leaf} &\neg a_{c}^b	&& \hspace{-3em}\text{initial $b$, every $c\in \colors(b)$}\\
\label{adm:union}
&a_{c}^b \leftrightarrow \hspace{-1.5em}\bigvee_{\substack{b'\in\children(b): c\in \colors(b')}} \hspace{-3em}a_{c}^{b'}	&& \hspace{-3em}\text{disjoint union $b$, every $c\in \colors(b)$}\raisetag{2.5em}\\
&a_c^b \leftrightarrow \hspace{-2.25em}\bigvee_{\text{if $b$ relabeling\ }c'\mapsto c} \hspace{-2.25em} a_{c'}^{b'} \vee \hspace{.15em}\bigvee_{\text{if\ }c\in \colors(b')}\hspace{-1.5em} a_{c}^{b'}
&& \text{relabeling $b$, every $c\in \colors(b),$}\notag\\[-1.6em]
\label{adm:relabel}
&&& b'\in\children(b)\\
&a_{c}^b \leftrightarrow a_{c}^{b'} \vee \hspace{-2.75em}\bigvee_{\text{if $b$ edge introduce\ }(c,c')}\hspace{-2.5em} e_{c'}^b 	&& \hspace{-1.5em}\text{edge introduce $b$, every $c{\in} \colors(b)$,} 
\notag\\[-1.65em]
\label{adm:edge-neg}
&&& \hspace{-1.5em}b'{\in}\children(b),\text{not introd.\ } (c',c)\raisetag{1em}\\
&a_{c}^b \leftrightarrow a_{c}^{b'} \wedge 
\neg e_{c'}^b 	&& \hspace{-1.5em}\text{edge introduce $b$, every $c{\in} \colors(b)$,} 
\notag\\[-.5em]
\label{adm:edge-pos}
&&& \hspace{-1.5em}b'\in\children(b),\text{introd.\ $(c',c)$}
\end{flalign}

In the root $rt$ we ensure that no color remains attacking. 
\begin{flalign}
\neg {a}_c^{rt} 	&& \text{for root colors\ } c\in \colors(rt)\label{adm:root}
\end{flalign}

\paragraph{Intuition}
The intuition behind this encoding is to remember whether some non-extension argument attacks an extension argument.
To this aim, Formulas~(\ref{adm:leaf}) initiate that no argument attacks an extension to begin with.
Then, Formulas~(\ref{adm:union})--(\ref{adm:relabel}) guide this information along the $k$ expression.
Our disjunctive encoding via colors enforces that each argument of a particular color has been considered.
Formulas~(\ref{adm:edge-neg})--(\ref{adm:edge-pos}) update the status of attacking argument depending on whether an argument of color $c$ attacks an argument of the extension color $c'$.
Finally, Formula~(\ref{adm:root}) forces that no color remains attacking to our extension.

\begin{restatable}[$\star$,Correctness]{theorem}{admcorrect}\label{thm:adm:correct}
Let~$F$ 
be an AF and~$\ptlong$ be a $k$-expression 
of~$F$. 
Then, the DDG
reduction~$\mathcal{R}_{\CAdm\rightarrow\CBSAT}$ is correct, that
is, $\CAdm$ on $F$ coincides with
$\CBSAT$ on $\mathcal{R}_{\CAdm\rightarrow\CBSAT}(F,\ptshort)$.
\end{restatable}
%

\begin{restatable}[$\star$,CW-Awareness]{theorem}{admcwaw}\label{thm:adm:cwaw}
Let~$F$ 
be an AF and~$\ptlong$ be a $k$-expression of~$F$. 
The DDG
reduction~$\mathcal{R}_{\CAdm\rightarrow\CBSAT}(F,\ptlong)$
constructs a SAT
instance~$\psi$ that linearly preserves the width, i.e.,
$\scw({\sinc{\psi})}\in\mathcal{O}(k)$.
\end{restatable}

\paragraph{\underline{Complete Extensions}}
For complete semantics, we compute admissible extensions such that there is no argument that could have been included. That is, there is no argument that is not in the extension but defended by it.
Moreover, we additionally need the information of whether a color has been defeated or not.
To achieve this, we reuse Equations~(\ref{stab:leaf})--(\ref{stab:edge}), (\ref{adm:leaf})--(\ref{adm:edge-pos}) and, on top, we compute whether
we could have included an argument via a collection $O$ of \emph{out variables}. 
These are encoded and propagated as follows.

\begin{flalign}
\label{out:leaf}&o_{c}^b \leftrightarrow \neg e_c^b	&& \hspace{-3.5em}\text{initial $b$, every $c\in \colors(b)$}\\
\label{out:union} &o_{c}^b \leftrightarrow \hspace{-1.5em}\bigvee_{\substack{b'\in\children(b): c\in \colors(b')}} \hspace{-3em} o_{c}^{b'}	&& \hspace{-3.5em}\text{disjoint union $b$, every $c\in \colors(b)$}\\
&o_c^b \leftrightarrow \hspace{-2.25em}\bigvee_{\text{if $b$ relabeling\ }c'\mapsto c} \hspace{-2em} o_{c'}^{b'} \vee \hspace{-.1em}\bigvee_{\text{if\ }c\in \colors(b')}\hspace{-1.5em} o_{c}^{b'}
&& \text{relabeling $b$, every $c\in \colors(b),$}\notag\\[-1.6em]
&&&b'\in\children(b)\label{out:relabel}\\
\label{out:att} &o_{c}^b \leftrightarrow o_{c}^{b'} \wedge \hspace{-3em}\bigwedge_{\text{if$b$ edge introduce\ }(c,c')}\hspace{-3.25em} \neg e_{c'}^b 	&& \hspace{-1.5em}\text{edge introduce $b$, every $c{\in} \colors(b)$,} 
\notag\\[-1.65em]
&&& \hspace{-1.5em}b'{\in}\children(b),\text{not introd.\,$(c',c)$}\raisetag{1em}\\
\label{out:defend} &o_{c}^b \leftrightarrow o_{c}^{b'} \hspace{-.4em}\wedge (c{\neq} c') \wedge d_{c'}^{\geq b} 	&& \hspace{-1.5em}\text{edge introduce $b$, every $c{\in} \colors(b)$,} 
\notag\\[-.3em]
&&& \hspace{-1.5em}b'{\in}\children(b),\text{introd.\ $(c',c)$}
\end{flalign}

To compute $d_{c}^{\geq b}$, we need to propagate the information of being defeated backwards, from the root towards the leaves.

\begin{flalign}
%
%
\label{def:rt}&{d}_{c}^{\geq{}rt} \leftrightarrow \hspace{-1.5em}\bigvee_{\substack{\text{if $rt$ edge introduce } (c',c)}} \hspace{-3em}e_{c'}^{rt}
&&\hspace{-2em}\text{for root colors $c{\in}\colors(rt)$}\\
\label{def:union}&{d}_{c}^{\geq{}b'} \leftrightarrow {d}_{c}^{\geq b} \vee \hspace{-3em}\bigvee_{\substack{\text{if $b'$ edge introduce } (c',c)}} \hspace{-3em}e_{c'}^{b'}
&& \text{non-relabeling $b$ with } b'\in\notag\\[-1.65em]
&&& \children(b), c\in \colors(b')\raisetag{1.15em}\\ 
\label{def:rel} &d_c^{\geq b'} \leftrightarrow \hspace{-2.25em}\bigvee_{\text{if $b$ relabeling\ }c\mapsto c'} \hspace{-2em} d_{c'}^{\geq b} \vee \hspace{-.35em}\bigvee_{\text{if\ }c\in \colors(b)}\hspace{-1.5em} d_{c}^{\geq{}b}
&& \text{relabeling $b$ with $b'\in$}\notag\\[-1.65em]
&&&\text{$\children(b), c\in \colors(b')$}\raisetag{1em}
\end{flalign}



In the root $rt$ no color remains attacking or out. 

\begin{flalign}
\label{out:root} \neg a_c^{rt}, \quad \neg {o}_c^{rt} 	&& \hspace{-1em}\text{for every\ } c\in \colors(rt)
\end{flalign}

\paragraph{Intuition} We encode whether some argument is incorrectly left out via a set $O$ of variables.
To achieve this, Formula~(\ref{out:leaf}) initiates  arguments not in an extension as candidates for incorrectly left out.
Formulas~(\ref{out:union})--(\ref{out:relabel}) guide this information along the $k$-expression.
Our disjunctive encoding here guarantees that we have considered each argument of any color.
Then, Formulas~(\ref{out:att})--(\ref{out:defend}) update the status of an argument depending on whether there is a valid reason to leave this out. 
Precisely, Formula~(\ref{out:att}) guarantees that if a color $c$ attacks an extension argument, it can not be left out since this must be defeated by the extension, due to the admissibility, and
Formula~(\ref{out:defend}) encodes that an argument was correctly left out, if there is some undefended attacker.
Then, Formulas~(\ref{out:root}) confirm that no argument is incorrectly left out.


\begin{restatable}[$\star$,Correctness]{theorem}{compcorrect}\label{thm:comp:correct}
Let~$F$ be an AF and~$\ptshort$ be a $k$-expression of~$F$.
%
The DDG reduction~$\mathcal{R}_{\CComp\rightarrow\CBSAT}$ is
correct, that is, $\CComp$ on $F$ coincides with $\CBSAT$ on
$\mathcal{R}_{\CComp\rightarrow\CBSAT}(F,\ptlong)$.
\end{restatable}

\begin{restatable}[$\star$,CW-Awareness]{theorem}{compcwaw}\label{thm:comp:cwaw}
Let~$F$ be an AF and~$\ptlong$ be a $k$-expression.
%
The reduction $\mathcal{R}_{\CComp\rightarrow\CBSAT}(F,\ptlong)$
constructs a SAT instance~$\psi$ that linearly preserves the width,
i.e., $\scw({\sinc{\psi})})\in\mathcal{O}(k)$.
\end{restatable}

\subsection{Semantics Using Subset-Maximization}\label{sec:second}

Before we turn our attention to maximization-based semantics, we require some meta-result on DNF matrices, which will substantially simplify our constructions below.

\begin{restatable}[$\star$]{lemma}{thmDNF}\label{thm:dnf}
Let $Q$ be a QBF with inner-most $\forall$ quantifier and matrix $\varphi \wedge \psi$, where $\varphi$ is in CNF and $\psi$ is in DNF. Assuming $k{=}\scw(\mathcal{G}^d_i(\varphi) \sqcup\mathcal{G}^d_i(\psi))$, there is a model-preserving DNF matrix~$\varphi'$ with $\scw(\mathcal{G}^d_i(\varphi''))\in\mathcal{O}(k)$.
\end{restatable}
\begin{proof}[Proof (Sketch)]
It suffices to encode the CNF $\varphi$ into a DNF formula $\varphi'$ along the $k$-expression.
The idea of our encoding is to guide the satisfiability of clauses along the $k$-expression, thereby keeping the information of whether a variable is assigned false 
or true. 
The resulting DNF matrix 
uses the formula $\varphi'$ instead.
Importantly, our encoding linearly preserves the directed incidence clique-width.
\end{proof}

\paragraph{\underline{Preferred Extensions}}
As before, we take an AF $F$, a $k$-expression $\ptlong$ of $F$,
and compute admissible extensions via the formula $\varphi_\CAdm$. 
However, we also need to keep track of subset-larger extension candidates.
So, we additionally create the formula $\varphi_\CAdm^*$, where starring refers to renaming every resulting variable $v$ by a new copy $v^*$.
%
It remains to design a structure-aware encoding of subset-larger extensions via starred variables. 
We construct the following CNF $\varphi_\pref$ (given as Equations~(\ref{pref:adm})--(\ref{pref:root})).

\begin{flalign}
\label{pref:adm}&\psi && \text{for every $\psi\in\varphi_\CAdm^*$}\\
%
&
s_{c}^b \leftrightarrow  {e^*_c}^{\hspace{-.1em}b} \wedge \neg e_c^b	&& \hspace{-2.5em}\text{initial $b$, create $a$ of color $c$}
\raisetag{2.25em} \label{pref:leaf}\\
&s_{c}^b \leftrightarrow \hspace{-2em}\bigvee_{\substack{b'\in\children(b): c\in \colors(b')}} \hspace{-3.25em}s_{c}^{b'}	&& \hspace{-4.25em}\text{disjoint union $b$, every $c\in \colors(b)$}\raisetag{2.45em}\label{pref:union}\\
&s_c^b \leftrightarrow \hspace{-2.25em}\bigvee_{\text{if $b$ relabeling\ }c'\mapsto c} \hspace{-2em} s_{c'}^{b'}\hspace{-.25em}\vee \hspace{.25em}\bigvee_{\text{if\ }c\in \colors(b')}\hspace{-1.75em} s_{c}^{b'}\label{pref:relabel}
&& \hspace{-.5em}\text{relabeling $b$, every $c\in \colors(b),$}\notag\\[-1.7em]
&&&\hspace{-.5em}b'\in\children(b)\\
&\hspace{-.1em}s_{c}^b \leftrightarrow s_{c}^{b'}\hspace{-.1em} 	&& \hspace{-7em}\text{edge introduce $b$, every $c{\in} \colors(b)$, }b'{\in}\children(b)\label{pref:edge}
\end{flalign}
We skip subset-larger extension counter candidates that are not in a superset relation to the candidate.

\begin{flalign}
\label{pref:subset}&e^b_c \rightarrow {e^*_c}^{\hspace{-.1em}b}&& \hspace{-.5em}\text{initial $b$, create $a$ of color $c$}
\end{flalign}

\noindent Further, for the root operation $rt$, we need to find an admissible extension that is subset-maximal, expressed as follows.

\begin{flalign}
\bigvee_{c\in \colors(rt)} s_{c}^{rt} && \text{for root operation }rt
\label{pref:root}
\end{flalign}

Then, the reduction $\mathcal{R}_{\CPref\rightarrow\CTQSAT}(F,\ptlong)$
constructs a QBF $\varphi_\CPref\dfn(\varphi_\CAdm \wedge \neg \big (\exists E^*, A^*, S. \varphi_\pref\big ))$,
which searches for an admissible extension (free variables) where there is no subset-larger
admissible extension ($\neg \exists$).
So, if there indeed is a larger extension than the candidate given via $e_a^b$ variables,
the QBF evaluates to false. 

If we bring this QBF into prenex normal form (shifting negation inside), we obtain an $\forall$-QBF with free variables whose matrix is of the form $\varphi_\CAdm \wedge \varphi_\pref$ with $\varphi_\CAdm$ in CNF and $\varphi_\pref$ in DNF.
This matrix can then be converted to DNF by Lemma~\ref{thm:dnf}.
We obtain the following~result.

\begin{restatable}[$\star$,Correctness]{theorem}{prefcorrect}\label{thm:pref:correct}
Let~$F$ be an AF and~$\ptlong$ be a $k$-expression of~$F$. 
%
The DDG reduction $\mathcal{R}_{\CPref\rightarrow\CTQSAT}$ is correct, that is, $\CPref$ on $F$ coincides with $\CTQSAT$ on $\mathcal{R}_{\CPref\rightarrow\CTQSAT}(F,\ptlong)$.
\end{restatable}

\begin{restatable}[$\star$,CW-Awareness]{theorem}{prefcwaw}\label{thm:pref:cwaw}
Let~$F$ be an AF and~$\ptlong$ $k$-expression.
The 
reduction~$\mathcal{R}_{\CPref\rightarrow\CTQSAT}(F,\ptlong)$
constructs a QSAT instance~$\psi$ that linearly preserves the
width,~i.e., $\scw{\sinc{\matr(\psi)}}\in\mathcal{O}(k)$.
\end{restatable}

\paragraph{\underline{Semi-Stable Extensions}}
To compute semi-stable extensions, we must maximize the range of arguments.
Interestingly, the range is computed, but via the information on colors
we can not directly access it. We construct the following CNF $\varphi_{\semi}$ (given as  equivalences).

\begin{flalign}
& \gamma && \hspace{-1.5em}\text{for every } \gamma\in \text{Formulas}~{(\ref{stab:defeat-leaf})}-{(\ref{stab:defeat-edge})} \notag\\
\label{semi:adm}&\psi && \text{for every $\psi\in\varphi^*_{\CAdm}$}\\
\label{semi:leaf}&s_{c}^b \leftrightarrow {d_c^*}^{\hspace{-.1em}b} \wedge \neg {d}_c^b	&& \hspace{-1.5em}\text{initial $b$, every $c\in \colors(b)$}
\raisetag{1.15em}\\
\label{semi:union}&s_{c}^b \leftrightarrow \hspace{-2em}\bigvee_{\substack{b'\in\children(b): c\in \colors(b')}} \hspace{-3em}s_{c}^{b'}	&& \hspace{-3.75em}\text{disjoint union $b$, every $c\in \colors(b)$}\raisetag{2.25em}\\
&s_c^b \leftrightarrow \hspace{-2.25em}\bigvee_{\text{if $b$ relabeling\ }c'\mapsto c} \hspace{-2.35em} s_{c'}^{b'} \vee \hspace{.35em}\bigvee_{\text{if\ }c\in \colors(b')}\hspace{-1.7em} s_{c}^{b'}
&& \text{relabeling $b$, every $c{\in} \colors(b),$}\notag\\[-1.6em]
\label{semi:relabel} &&&b'\in\children(b)\\
&s_{c}^b \leftrightarrow (s_{c}^{b'} \vee {d_c^*}^{\hspace{-.1em}b}) \wedge \neg d_c^b 	&& \hspace{-.5em}\text{edge introduce $b$, every $c\in$}\notag\\[-.55em] 
\label{semi:edge} &&&\hspace{-.5em}\colors(b), b'\in\children(b)
\end{flalign}

For the root $rt$, we keep extensions of strictly larger range. 

\begin{flalign}
d_{c}^{rt}\rightarrow {d^{*}_c}^{\hspace{-.2em}rt}, \quad \bigvee_{c\in \colors(rt)} s_{c}^{rt}  && \text{for root operation } rt
\label{semi:root}\raisetag{1.25em}
\end{flalign}

The reduction $\mathcal{R}_{\CSemiSt\rightarrow\CTQSAT}(F,\ptlong)$
constructs a QBF $(\varphi_\CAdm \wedge(\forall E^*, D^*, A^*, S. 
\neg\varphi_{\semi}))$,
where $\neg\varphi_{\semi}$ is in DNF, but  $\varphi_\CAdm$ is in CNF. As above, Lemma~\ref{thm:dnf} converts the matrix into DNF as desired.


\begin{restatable}[$\star$,Correctness]{theorem}{stagcorrect}\label{thm:stag-correct}
Let~$F$ be an AF and~$\ptlong$ be a $k$-expression of~$F$.
%
Then, the DDG reduction~$\mathcal{R}_{\CSemiSt\rightarrow\CTQSAT}$ is
correct, that is, $\CSemiSt$ on $F$ coincides with $\CTQSAT$ on
$\mathcal{R}_{\CSemiSt\rightarrow\CTQSAT}(F,\ptlong)$.
\end{restatable}

\begin{restatable}[$\star$,CW-Awareness]{theorem}{stagcwaw}\label{thm:stag:cwaw}
Let~$F$ be an AF and~$\ptlong$ a
$k$-expression. 
%
The
reduction~$\mathcal{R}_{\CSemiSt\rightarrow\CTQSAT}(F,\ptlong)$
constructs a QSAT instance~$\psi$ that linearly preserves the width,
i.e., $\scw({\sinc{\matr(\psi)}})\in\mathcal{O}(k)$.
\end{restatable}

\paragraph{\underline{Stage Extensions}}
The reduction $\mathcal{R}_{\CStage\rightarrow\CTQSAT}(F,\ptlong)$
constructs a QBF $\forall E^*, D^*, S. (\varphi_\CConf \wedge \neg\varphi_{\stag})$,
where $\varphi_{\stag}$ is a CNF
comprising $\psi$ for every $\psi\in \varphi_{\#\conf}$
as well as Equations~(\ref{semi:leaf})--(\ref{semi:root}).
Correctness and CW-Awareness works as above,  
so we obtain the following upper bounds.

%
%

\begin{theorem}[Runtime-UBs]\label{thm:counting-ub}
Let $F=(A,R)$ be an AF of size $n$ and directed clique-width $k$, For a semantics $\sigma$, the problem $\#\sigma$ can be solved in time
\begin{itemize}
\item $2^{{\mathcal O}(k)}\cdot \poly(n)$ for $\sigma\in\{\stab,\adm,\comp\}$.
\item $2^{2^{{\mathcal O}(k)}}\cdot \poly(n)$ for $\sigma\in\{\pref, \semi,\stag\}$.
\end{itemize}
\end{theorem}

\subsection{Credulous and Skeptical Reasoning}
The preceding reductions can be extended to determine credulous  
and skeptical acceptance  
of an argument. 
Let $F$ 
be an AF and $\sigma$ be a semantics.
To solve credulous acceptance $c_\sigma$ for $a$, we append ``$e_a$'' to each formula $\varphi_{\#\sigma}$ where $e_a\in E$ is the extension variable corresponding to argument $a$.
Then, each satisfying assignments for $\varphi_\sigma\land e_a$ yields an extension (via extension variables) containing $a$.
Moreover, to solve skeptical acceptance $\skept\sigma$, we add ``$\neg e_a$'' to $\varphi_{\#\sigma}$ and flip the answer in polynomial time,~i.e., $\skept\sigma$ is true for $a$ if and only if there is no satisfying assignment for $\varphi_{\#\sigma}\land \neg e_a$.

We observe that correctness and CW-awareness in both cases for each semantics follows from proofs of corresponding ``Correctness'' and ``CW-Awareness'' theorems,~e.g., Thm.~\ref{thm:stab:correct} and Thm.~\ref{thm:stab:cwaw} for stable semantics.

\begin{theorem}[Runtime-UBs]\label{thm:cred-skep-ub}
Let $F=(A,R)$ be an AF of size $n$ and directed clique-width $k$, For a semantics $\sigma$, the problem $\skept\sigma$ can be solved in time
\begin{itemize}
\item $2^{{\mathcal O}(k)}\cdot \poly(n)$ for $\sigma\in\{\stab,\adm,\comp\}$.
\item $2^{2^{{\mathcal O}(k)}}\cdot \poly(n)$ for $\sigma\in\{\pref,\semi,\stag\}$.
\end{itemize}
$\cred\sigma$ behaves similarly, but $\cred\pref$ can be solved via $\cred\adm$.
\end{theorem}


\section{Lower Bounds: Can We Improve?}

It turns out that we can not significantly improve most of our reductions.
Indeed, we can create a clique-width-aware reduction from 3SAT to asking whether some argument is included in an admissible extension.

\begin{restatable}[$\star$,Admissible CW-LB]{theorem}{admlb}\label{thm:lb}
Unless ETH fails, we can not decide for an AF $F=(A,R)$ of directed clique-width $w$ in time $2^{o(w)}\cdot\poly(|A|+|R|)$ whether there exists an admissible extension
of $F$ containing argument $a$.
\end{restatable}

We can easily extend this to other semantics.
Indeed, the lower bound immediately carries over to other semantics.

\begin{corollary}[Stable/Complete CW-LB]\label{cor:stabcomp}
Under ETH we can not decide for an AF $F=(A,R)$ of directed clique-width $w$ in time $2^{o(w)}\cdot\poly(|A|+|R|)$ whether there is a stable/complete extension of $F$ containing~$a$.
\end{corollary}

\noindent The bounds can be extended to second-level extensions.

\begin{restatable}[Preferred/Semi-Stable/Stage CW-LB]{proposition}{preflb}\label{thm:slb}
Under ETH we can not decide for an AF $F=(A,R)$ of directed clique-width $w$ in time $2^{2^{o(w)}}\cdot\poly(|A|+|R|)$ whether there exists a preferred (semi-stable/stage) extension
of $F$ that does not contain $a$ (contains $a$).
\end{restatable}

These results indicate that we cannot significantly improve our reductions.
Indeed, for solving the second-level semantics, a reduction to SAT is expected to be insufficient.

\section{Conclusion}
%
Our results 
answer
whether we can efficiently encode knowledge representation and
reasoning (KRR) formalisms into (Q)SAT while respecting the
clique-width.  Table~\ref{tab:results} provides a comprehensive
overview.
Our directed decomposition guided (DDG) reductions 
based on $k$-expressions make existing results on clique-width for SAT
(Proposition~\ref{prop:numsatruntime}) and QSAT
(Proposition~\ref{prop:mengel}) accessible to abstract
argumentation. Using these results and our novel DDG reduction, we
establish efficient solvability when exploiting clique-width for all
argumentation semantics and the problems of extension existence,
acceptance, and counting.
%
%
Finally, we prove that we cannot significantly improve under
reasonable assumptions.
\postrebuttal{Our approach remains effective even when the attack graph includes large cliques or complete bipartite structures, where other parameters (e.g., treewidth) fail.}
%
%
%
%
%
%
%

We see various 
directions for future works.  We are
interested whether these results can be extended to other KRR
formalisms such as abductive reasoning, logic-based argumentation,
answer-set programming, and many more.
We expect that this work might be useful for simple
classroom-type proofs of Courcelle's theorem for clique-width, see~\cite{BannachHecher25}.
%
Since our reductions preserve the solutions bijectively, we are
interested in enumeration complexity as well.
Finally, generalized or orthogonal versions of clique-width such as
modular incidence treewidth and symmetric incidence clique-width.


\section*{Acknowledgment}
Research was partly funded by the Austrian Science Fund (FWF), grant J4656, the French Agence nationale de la recherche (ANR), grant ANR-25-CE23-7647, the Deutsche Forschungsgemeinschaft (DFG, German Research Foundation), grant TRR 318/1 2021 – 438445824, the Ministry of Culture and Science of North Rhine-Westphalia (MKW NRW) within project WHALE (LFN 1-04) funded under the Lamarr Fellow Network programme, and the Ministry of Culture and Science of North Rhine-Westphalia (MKW NRW) within project SAIL, grant NW21-059D. 
Part of the research was carried out while Hecher was a postdoc at MIT and while he was visiting the Simons institute for the theory of computing (part of the program \emph{Logic and Algorithms in Database Theory and AI}).
Fichte was funded by ELLIIT funded by the Swedish government.

\clearpage

\bibliographystyle{abbrv}
\bibliography{main}

\cleardoublepage

\section{Detailed Examples}
\paragraph{Admissible Extensions}
\begin{example}\label{ex:adm}
	Consider the AF from Example~\ref{ex:running} and the corresponding parse tree from Figure~\ref{fig:excw}. For this example, we consider the formulas added for the right side of the tree, i.e. operations $3,5,8,9,10$. We already discussed the formulas added for Equations~(\ref{stab:leaf})-(\ref{stab:edge}) in Example~\ref{ex:stable}. Consider now the formulas added for Equations~(\ref{adm:leaf})-(\ref{adm:edge-pos}). For initial operations $9,10$ by Equation~(\ref{adm:leaf}), we add the two formulas
	\begin{flalign*}
		&\neg {a}_{\cvu{1}}^{9}, \\
		&\neg {a}_{\cvd{2}}^{{10}}.
	\end{flalign*}
	For the disjoint union operation $8$, by Equation~(\ref{adm:union}), we add the formulas
	\begin{flalign*}
		&{a}_{\cvu{1}}^{8} \leftrightarrow {a}_{\cvu{1}}^{9}, \\
		&{a}_{\cvd{2}}^{8} \leftrightarrow {a}_{\cvd{2}}^{10}.
	\end{flalign*}
	For the edge-introducing operation $5$, by Equation~(\ref{adm:edge-neg}) and~(\ref{adm:edge-pos}), we add the formulas
	\begin{flalign*}
		&{a}_{\cvu{1}}^{5} \leftrightarrow {a}_{\cvu{1}}^{8} \lor e_{\cvd{2}}^{5}, \\
		&{a}_{\cvd{2}}^{5} \leftrightarrow {a}_{\cvd{2}}^{8} \land \neg e_{\cvu{1}}^{5}.
	\end{flalign*}
	Finally, for the relabeling operation $3$, by Equation~(\ref{adm:relabel}), we add the formulas
	\begin{flalign*}
		&{a}_{\cvd{2}}^{3} \leftrightarrow {a}_{\cvd{2}}^{5}, \\
		&{a}_{\cvt{3}}^{3} \leftrightarrow {a}_{\cvu{1}}^{5}.
	\end{flalign*}
\end{example}

\paragraph{Complete Extensions}
\begin{example}\label{ex:complete}
	Consider the AF from Example~\ref{ex:running} and the corresponding parse tree from Figure~\ref{fig:excw}. For this example, we consider the formulas added for the right side of the tree, i.e. operations $3,5,8,9,10$. We already discussed the formulas added for Equations~(\ref{stab:leaf})--(\ref{stab:edge}) in Example~\ref{ex:stable} and those added for Equations~(\ref{adm:leaf})--(\ref{adm:edge-pos}) in Example~\ref{ex:adm}. Consider now the formulas added for Equations~(\ref{out:leaf})--(\ref{out:defend}). For initial operations $9,10$ by Equation~(\ref{out:leaf}), we add the two formulas
	\begin{flalign*}
		&{o}_{\cvu{1}}^{9} \leftrightarrow \neg e_{\cvu{1}}^{9}, \\
		&{o}_{\cvd{2}}^{10} \leftrightarrow \neg e_{\cvd{2}}^{10}.
	\end{flalign*}
	For the disjoint union operation $8$, by Equation~(\ref{out:union}), we add the formulas
	\begin{flalign*}
		&{o}_{\cvu{1}}^{8} \leftrightarrow {o}_{\cvu{1}}^{9}, \\
		&{o}_{\cvd{2}}^{8} \leftrightarrow {o}_{\cvd{2}}^{10}.
	\end{flalign*}
	For the edge-introducing operation $5$, by Equation~(\ref{out:att}) and~(\ref{out:defend}), we add the formulas
	\begin{flalign*}
		&{o}_{\cvu{1}}^{5} \leftrightarrow {o}_{\cvu{1}}^{8} \land e_{\cvd{2}}^{5}, \\
		&{o}_{\cvd{2}}^{5} \leftrightarrow {o}_{\cvd{2}}^{8} \land (\cvu{1} \neq \cvd{2}) \land d_{\cvu{1}}^{\geq 5}.
	\end{flalign*}
	Finally, for the relabeling operation $3$, by Equation~(\ref{adm:relabel}), we add the formulas
	\begin{flalign*}
		&{a}_{\cvd{2}}^{3} \leftrightarrow {a}_{\cvd{2}}^{5}, \\
		&{a}_{\cvt{3}}^{3} \leftrightarrow {a}_{\cvu{1}}^{5}.
	\end{flalign*}
\end{example}

\paragraph{Preferred Extensions}
\begin{example}\label{ex:pref}
	Consider the AF from Example~\ref{ex:running} and the corresponding parse tree from Figure~\ref{fig:excw}. For this example, we consider the formulas added for the right side of the tree, i.e. operations $3,5,8,9,10$. We already discussed formula $\varphi_\CAdm$ in Example~\ref{ex:adm}. Formula $\varphi_\CAdm^*$ is defined analogue by renaming every variable $v$ in $\varphi_\CAdm$ to $v^*$ and thus also does not need to be discussed in this Example. We instead discuss the formulas added for Equations~(\ref{pref:leaf})--(\ref{pref:edge}). For initial operations $9,10$, by Equation~(\ref{pref:leaf}), we add the two formulas
	\begin{flalign*}
		&{s}_{\cvu{1}}^{9} \leftrightarrow e_{\cvu{1}}^{*9} \land \neg e_{\cvu{1}}^{9}, \\
		&{s}_{\cvd{2}}^{10} \leftrightarrow e_{\cvd{2}}^{*10} \land \neg e_{\cvd{2}}^{10}.
	\end{flalign*}
	For the disjoint union operation $8$, by Equation~(\ref{pref:union}), we add the formulas
	\begin{flalign*}
		&{s}_{\cvu{1}}^{8} \leftrightarrow {s}_{\cvu{1}}^{9}, \\
		&{s}_{\cvd{2}}^{8} \leftrightarrow {s}_{\cvd{2}}^{10}.
	\end{flalign*}
	For the edge-introducing operation $5$, by Equation~(\ref{pref:edge}), we add the formulas
	\begin{flalign*}
		&{s}_{\cvu{1}}^{5} \leftrightarrow {s}_{\cvu{1}}^{8}, \\
		&{s}_{\cvd{2}}^{5} \leftrightarrow {s}_{\cvd{2}}^{8}.
	\end{flalign*}
	Finally, for the relabeling operation $3$, by Equation~(\ref{pref:relabel}), we add the formulas
	\begin{flalign*}
		&{s}_{\cvd{2}}^{3} \leftrightarrow {s}_{\cvd{2}}^{5}, \\
		&{s}_{\cvt{3}}^{3} \leftrightarrow {s}_{\cvu{1}}^{5}.
	\end{flalign*}
\end{example}

\section{Detailed Proofs}

\paragraph{Stable Extensions}

\stabcorrect*
\begin{proof}
	``$\Longrightarrow$'': Let $S\in
	\stab(F)$. 
	We prove that $\varphi_\CStab$ is satisfiable and the satisfying
	assignments coincide.
	To this aim, we construct a unique satisfying assignment $\alpha$
	over variables $E$ and $D$ in a bottom-up fashion, as follows.

	First, we set $\alpha(e_a)=1$, and only if, $a\in S$. 
	Then, we assign the value of $\alpha$ for remaining variables according to the following rules. We let $\alpha(e_c^b)=1$ for each color $c$ and
	operation~$b$ in the $k$-expression, 
	if and only if,
	\begin{itemize}
		\item[(C1)] $\alpha(e_a)=1$\\
		\mbox{~}\hfill if $a\in S$, such that $b=c(a)$ is an initial $k$-graph\\
		\mbox{~}\hfill for argument $a$ colored with~$c$;
		\item[(C2)] $\alpha(e_c^{b'})=1$ for some $b'\in\child(b)$ and
		$c\in\colors(b')$,\\
		\mbox{~}\hfill if $b=\oplus$ and $c\in \colors(b)$;
		\item[(C3)] either $\alpha(e_{c'}^{b'})=1$ or
		$\alpha(e_{c}^{b'})=1$,\\
		\mbox{~}\hfill if $b= \rho_{c'\mapsto c}$ with $c\in \colors(b)$ and
		$c'\in \colors(b')$;
		\item[(C4)] $\alpha(e_c^{b'})=1$,\\
		\mbox{~}\hfill if $b= \eta_{c',c}$.
	\end{itemize}
	
	\noindent Furthermore, we set $\alpha(d_c^b)=1$ for every color $c$
	and operation~$b$ of the $k$-expression, 
	if and only if,
	\begin{itemize}
		\item[(C5)] $\alpha(e_c^b)=1$,\\
		\mbox{~}\hfill if $b=c(a)$ is an initial $k$-graph with $c\in\colors(b)$;
		\item[(C6)] $\alpha(d_c^{b'})=1$ for each~$b'\in\child(b)$
		s.t. $c\in\colors(b')$,\\
		\mbox{~}\hfill if $b= \oplus$;
		\item[(C7)] $\alpha(d_{c'}^{b'})=1$ and $\alpha(d_c^{b'})=1$,\\
		\mbox{~}\hfill if $b= \rho_{c'\mapsto c}$ with $c\in\colors(b)\cap \colors(b')$;
		\item[(C8)] either $\alpha(d_c^{b'})=1$ or $\alpha(e_{c'}^{b})=1$,\\
		\mbox{~}\hfill if $b= \eta_{c',c}$ with $c\in \colors(b)$.
	\end{itemize}
	
	\noindent Next, we  prove that $\alpha$ satisfies $\varphi_{\#\Stab}$.\\
	It is easy to observe by construction that
	Formulas~(\ref{stab:leaf})--(\ref{stab:edge}a) are satisfied
	by~$\alpha$.
	Regarding Formulas~(\ref{stab:edge}b), observe that $S$ is
	conflict-free.  Suppose to that contrary that there is some operation
	$b= \eta_{c',c}$ such that $\alpha(e_c^b)=1$ and
	$\alpha(e_{c'}^b)=1$.
	Due to construction (C4), this is the case iff there is some child,
	or descendant, $b'$ of $b$ such that $\alpha(e_c^{b'})=1$ and
	$\alpha(e_{c'}^{b'})=1$ (possibly $b'\neq b$).
	However, the value for $\alpha(e_c^{b'})$ for each color~$c$ and
	operation $b'$ is simply propagated from some argument~$a\in S$ of
	color~$c$ in the initial $k$-graph $b=c(a)$.
	But this is a contradiction to the conflict-freeness of $S$, since
	in some child (descendant) operation of $b$, there are arguments~$a$
	and $a'$ of color~$c$ and $c'$ with $(a,a')\in R$.
	Therefore, either $\alpha(e_c^b)=0$ or $\alpha(e_{c'}^b)=0$ for
	every $b=\eta_{c',c}$ with $c\in\colors(b)$.

	It can be similarly observed that $\alpha$ satisfies
	Formulas~(\ref{stab:defeat-leaf})--(\ref{stab:defeat-edge}) by
	construction.
	It remain to prove that $\alpha$ satisfies
	Formula~(\ref{stab:defeat-root}).
	Suppose to the contrary that there is some color $c$ such that
	$\alpha(d_c^{rt})=0$.  Due to
	Formulas~(\ref{stab:defeat-union})--(\ref{stab:defeat-relabel}),
	this implies that $\alpha(d_c^{b'})=0$ for some operation
	$b'\in\children(b)$ with $b=\oplus$ or $b=\rho_{c',c}$ and
	$c\in\colors(b)$.
	We follow such a \emph{branch} in the $k$-expression 
	via child operations~$b'\in\children(b)$ where $\alpha(d_c^{b'})=0$.
	This leads us to some initial $k$-graph $b_0$ using the color $c$,
	such that $\alpha(e_c^{b_0})=0$ due to
	Formulas~(\ref{stab:defeat-leaf})
	(we prove our claim for the color $c$, whereas the case when $c$ appears in a relabeling $c^*\mapsto c$ on this branch follows analogous reasoning for $c^*$ additionally).
	In particular, either this branch does not contain an operation $b$
	and color $c'$ with $c,c'\in\colors(b)$ such that $b=\eta_{c',c}$,
	or we
	have 
	$\alpha(e_{c'}^b)=0$ for such an operation $b=\eta_{c',c}$ due to
	Formulas~(\ref{stab:defeat-edge}).
	This implies that, there is no argument~$a$ of color $c$ in $S$ due
	to Formulas~(\ref{stab:leaf}) and~(\ref{stab:defeat-leaf}) and no
	argument of color~$c$ is attacked by any argument in~$S$ due to
	Formulas~(\ref{stab:defeat-edge}).
	But this leads to a contradiction, since there is at least one
	argument in~$F$ of color~$c$, and every argument is either in~$S$ or
	attacked by~$S$ due to the stability of $S$.
	This completes our claim in this direction and hence
	$\alpha$ satisfies $\varphi_{\#\Stab}$.

	It remains to prove that $\alpha$ is unique for the set~$S$.
	To this aim, let $\beta\neq\alpha$ be another assignment over
	$E\cup D$ such that $\beta$ satisfies $\varphi_\CStab$ and
	$\{a\mid a\in S\} = \{a\mid \beta(e_a)=1\}$.
	Observe that, by definition $\beta(d_c^{rt})=1$ for a color $c$ if
	and only if either 
	(i) there is an edge introduction operation~$b= \eta_{c'c}$ with $\beta(e_{c'}^{b})=1$, 
	or (ii) $\beta(e_c^b)=1$ in the initial $k$-graph~$b$,
	or (iii) there is a relabeling operation $\rho_{c^*\rightarrow c}$ and (i) and (ii) are true for $c^*$.
	Moreover, $\beta(e_c^b)=1$ if and only if
	$\beta(e_a)=1$ for the argument $a$ such that $b= c(a)$ is the initial operation.
	Since $\beta$ assigns variable in~$E$ exactly as $\alpha$, the same
	applies to variables in~$D$.
	However, this leads to a contradiction and we obtain
	$\beta = \alpha$.  Therefore, for every stable extension~$S$ in~$F$,
	there is exactly one satisfying assignment $\alpha$ to variables
	$E\cup D$.
	
	``$\Longleftarrow$'': Suppose there is an assignment~$\theta$ over
	variables in~$E$ and $D$ such that $\theta$ satisfies
	$\varphi_\CStab$.
	Then, we construct a stable extension as $S=\{a\mid \theta(e_a)=1\}$
	for $F$.
	%
	Due to Formulas~(\ref{stab:defeat-root}), it follows that
	$\theta(d_c^{rt})=1$ for each $c\in\colors(rt)$.
	Since $\theta$ satisfies Formulas~(\ref{stab:edge}) for each edge
	introduction operation $b$ and color $c$, the set~$S$ is conflict
	free.
	Suppose to the contrary, there are arguments $a,a'\in S$ with
	$(a,a')\in R$.
	Then, let $b$ be the edge insertion operation for colors $c,c'$ with
	$col(a)=c$ and $col(a')=c'$ in $b$.
	Since, $\theta(e_a)= 1= \theta(e_{a'})$, we have that
	$\theta(e_c^{b'})=1=\theta(e_c^b)$ for the
	operation~$b'\in\children(b)$.
	Then, due to Formulas~(\ref{stab:edge}a), we have that
	$\theta(e_c^{b})=1=\theta(e_{c'}^b)$.
	But this leads to a contradiction since $\theta$ can not satisfy
	Formulas~(\ref{stab:edge}a) and Formulas~(\ref{stab:edge}b) at the
	same time.  Consequently, $S$ is conflict-free.
	
	Next, we prove that every argument~$a\in A\setminus S$ is attacked
	by some argument~$a'\in S$.
	Again, suppose to the contrary there is some argument $a\in A$ such
	that neither $a\in S$, nor is there any $a'\in S$ with
	$(a',a)\in R$.
	Let $col(a)=c$ in the root operation, i.e., the rightmost operation
	in the $k$-expression.
	Clearly, $\theta(d_c^{rt})=1$ since $\theta$ satisfies
	Formulas~(\ref{stab:defeat-root}).
	From
	Formulas~(\ref{stab:defeat-union})--(\ref{stab:defeat-relabel}), we
	have that $\theta(d_c^{b})=1$ for every descendant $b$ of $rt$ such
	that $c\in \colors(b)$.
	Assume wlog that $col(a)=c$ in the initial operation b. The
	alternative case follows analogous reasoning if $col(a)=c^*$ via the
	relabeling~$c^*\mapsto c$.
	This implies that we have one of the following two cases.

	In the first case, we have $\theta(e^b_{c})=1$ in the $b=c(a)$ operation for creating
	argument~$a$ of color~$c$ due to
	Formula~(\ref{stab:defeat-leaf}).
	But this is a contradiction to Formula~(\ref{stab:leaf}), since
	$a\not\in S$ and hence $\theta(e_a)=0$.
	In the second case, there is some edge introduction operation~$b$
	with the edge~$(c',c)$ and $\theta(e^b_{c'})=1$ due to
	Formula~(\ref{stab:defeat-edge}).
	Then, applying the same reasoning as in the only-if direction
	to~$c'$, we have that there is some argument~$a'\in S$ due to
	$\theta(e_{a'})=1$ according to Formula~(\ref{stab:leaf}), such that $a'$ has color $c'$ in $b$ and $(a',a)\in R$, since $b=\eta_{c',c}$ introduces an edge.
	But this again leads to a contradiction since there is no
	argument~$a'\in S$ attacking~$a$.
	Consequently, for every satisfying assignment~$\theta$ to variables
	$E\cup D$ of $\varphi_\CStab$, there is exactly one stable extension
	$S$ in $F$.
	This concludes correctness and establishes the claim.
\end{proof}

\stabcwaw*
\begin{proof}
	Given an AF~$F$ 
	and a $k$-expression~$\ptlong$ 
	of~$F$, we construct a $k'$-expression~$\ptlong'$ 
	of~$\sinc{\psi}$ as follows.
	We first specify additional colors needed to construct a $k'$-expression of $\psi$.
	For each color~$c$, we need an
	\emph{extension-version}~$e_c$ and a
	\emph{defeat-version}~$d_c$ of $c$.
	Moreover, for each extension~($e_c$) and defeat
	color~($d_c$), we additionally need two more copies, called
	\emph{child-versions} $ec_c$ and
	$dc_c$ to 
	add clauses corresponding to the child-operation
	$b'$ of a operation $b\in\ptlong$ and \emph{expired-versions}
	$ex_c$ and
	$dx_c$ to simulate the effect that all the edges between certain
	variables and their clauses have been added. 
	This is required, since otherwise, we we will keep adding edges
	from the child-versions in future, which is undesirable.
	Finally, to handle both positive and negative literals in clauses, we replicate current and the child versions of each colors each $x\in \{e_c,ec_c,d_c,dc_c\}$ twice, denoted as $x^+$ for positive and $x^-$ for negative literals. 
	%
	For clauses, we use two additional colors, called \emph{clause-making}
	($cm$) and \emph{clause-ready} ($cr$) to add 
	edges between clauses
	and their respective literals, since we are in the setting of
	incidence graphs.
	Intuitively, when adding 
	edges between clauses and their respective
	variables, we initiate a clause $C$ to be of color $cm$ and its literals $x$ to be of the appropriate color $x^+$ or $x^-$.
	Then, we draw directed edges between the clause $C$ and its positive and negative literals.
	Once the edges for $C$ have been added, 
	we relabel it to $cr$ to
	avoid any further edges to or from $C$.
	Similarly, when going from an operation~$b'$ to its parent $b$, we
	change the labels of extension and defeated variables in $b'$ to
	their child-versions.

	
	We first need to switch $\varphi_\CStab$ from Formulas~(\ref{stab:leaf})--(\ref{stab:defeat-root}) into a collection of clauses.
	This can be easily achieved via iteratively replacing each formula in $\varphi_\CStab$ using the following equivalences: (1) $\alpha\leftrightarrow\beta\equiv (\neg\alpha\lor\beta)\land (\neg\beta\lor\alpha)$, (2) $\neg (x_1\lor\dots\lor x_n) \lor C \equiv \bigwedge_{i\leq n}(\neg x_i\lor C)$, and (3) $\neg (x_1\land\dots\land x_n) \lor C \equiv (\bigvee_{i\leq n}\neg x_i\lor C)$.
	As a result, Formulas~(\ref{stab:leaf})--(\ref{stab:defeat-root}) (and hence $\varphi_\CStab$) can be converted into a CNF.
	Observe that, for each operation $b$: Formula~(\ref{stab:leaf}) and (\ref{stab:edge}a) each yields two clauses and Formula~(\ref{stab:edge}b) is already a clause.
	Likewise, for each operation $b$, the number for clauses corresponding to Formula~(\ref{stab:union}) is bounded by the children of $b$ via $|\children(b)|+1$, and those for Formula~(\ref{stab:relabel}) are bounded by three (since each disjunct is just a Boolean condition).
	The similar argument applies to Formulas~(\ref{stab:defeat-leaf})--(\ref{stab:defeat-edge}), thus giving a bounded many clauses in each case.
	This has the effect that our final expression $\ptlong'$ constructed below has its size polynomial in the size of $\ptlong$.
	
	In the following, we present our construction of the expression for only one representative clause from each formula whereas analogous construction applies to the remaining cases.
	For each operation in the $k$-expression $\ptlong$ of $F$ and the corresponding clause $C\in \varphi_\CStab$: we explain how to use additional colors described above to return a sub-expression that adds all edges between $C$ and its literals.
	Intuitively, the edges for each clause from $\varphi_\CStab$ and its corresponding literals can be added in a turn by turn fashion using auxiliary colors.
	Recall that we add directed edges between clauses and literal as: $(C,x)$ if $x\in C$ and $(x,C)$ if $\neg x\in C$ (see definition of directed incidence graph for CNF).
	
	
	Each operation $b\in\ptlong$ is translated into a sub-expression in $\ptlong'$. 
	For brevity, we only outline how sub-expressions corresponding to each operation are added.
	Precisely, we construct~$\ptshort'$ 
	as follows:
	\begin{itemize}
		\item \textbf{For initial}, $b=c(a)$: we create the following sub-expression: 
		\begin{enumerate}
			\item consider the clause $C_{b,c,a}\dfn(\neg e_c^b\lor e_a)$ corresponding to Formula~(\ref{stab:leaf}) for the given operation $b$, color $c$, and argument $a$. Initiate $c^+(e_a)$, $e^-_c(e^b_c)$ and $cm(C_{b,c,a})$.
			\item take disjoint  union ($\oplus$) of the operations from 1 above.
			\item add 
			edges for the considered clause of Formula~(\ref{stab:leaf}): $\eta_{cm, c^+}$ and $\eta_{e^-_c,cm}$.
			\item relabel $C_{b,c,a}$ as completed: $\rho_{cm\mapsto cr}$.
			\item consider the clause $C_{b,c}=(d_c^b\lor \neg  e_c^b)$ corresponding to  Formula~(\ref{stab:defeat-leaf}) for the given operation $b$ and color $c$.
			Initiate $d^+_c(d^b_c)$ and $cm(C_{b,c})$. 
			\item take disjoint union ($\oplus$) of the operations from 4 and 5 above.
			\item add 
			edges for the considered clause $C_{b,c}$ of Formula~(\ref{stab:defeat-leaf}): $\eta_{cm, d^+_c}$ and $\eta_{e^-_c,cm}$.
			\item relabel $C_{b,c}$ as completed:  $\rho_{cm\mapsto cr}$.
			\item Re-iterate for any remaining clauses corresponding to Formula~(\ref{stab:leaf}) or (\ref{stab:defeat-leaf}). Once all clauses have been considered, relabel extension and defeat colors to their child-version: $\rho_{e^o_c\mapsto ec^o_c}, \rho_{d^o_c\mapsto dc^o_c}$ for $o\in\{+,-\}$.
		\end{enumerate}
		\item \textbf{For disjoint union}  $b=\oplus_{b'\in\child(b)} b'$, we create the following sub-expression:
		\begin{enumerate}
			\item 
			consider a clause $C_{b,c}$  corresponding to Formula~(\ref{stab:union}) 
			for the given operation $b$ and color $c\in\colors(b)$. Initiate $e^o_c(e_c^b), cm(C_{b,c})$ where $o= +/-$ depending on whether $C_{b,c}$ contains $e_c^b$ or $\neg e_c^b$, respectively. 
			Moreover, relabel (if not already) the color of $e_c^{b'}$ from the child operation $b'$ so that we have: $ec^p_c(e_c^{b'})$ where $p=+$ if $C_{b,c}$ contains $e_c^{b'}$, and $p=-$ if it contains $\neg e_c^{b'}$.
			\item take disjoint union ($\oplus$) of the operations used in 1 above. 
			\item add 
			edges for the considered clause $C_{b,c}$ of Formula~(\ref{stab:union}): $\eta_{cm, e^o_c}$ (resp.,  $\eta_{e^o_c, cm}$) and $\eta_{ec^{p}_c, cm}$ ($\eta_{cm, ec^{p}_c}$) accordingly, depending on the current colors $o$ and $p$ for literals in 1 above.
			\item relabel $C_{b,c}$ as completed: $\rho_{cm\mapsto cr}$
			\item consider a clause $C_{b,c}$ corresponding to Formula~(\ref{stab:defeat-union}) for given $b$ and $c$. Initiate $d^o_c(d_c^b)$ and $cm(C_{b,c})$ where $o= +/-$ depending on whether $C_{b,c}$ contains $d_c^b$ or $\neg d_c^b$, respectively.
			Moreover, relabel (if not already) the color of $d_c^{b'}$ from the child operation $b'$ so that we have: $dc^p_c(d_c^{b'})$ where $p=+/-$ depending on whether $C_{b,c}$ contains $d_c^{b'}$ or $\neg d_c^{b'}$, respectively.
			\item take disjoint union ($\oplus$) of operations form 4 and 5 above.
			\item add 
			edges for the considered clause $C_{b,c}$ of Formula~(\ref{stab:defeat-union}): $\eta_{d^o_c,cm}$ (resp.,  $\eta_{cm, d^o_c}$) and $\eta_{cm, dc^p_c}$ ($\eta_{dc^p_c, cm}$) accordingly, depending on the current colors $o$ and $p$ of literals in 6 above.
			\item relabel $C_{b,c}$ as completed: $\rho_{cm\mapsto cr}$.
			\item Re-iterate for any remaining clauses corresponding to Formulas~(\ref{stab:union}) or (\ref{stab:defeat-union}). Once all clauses have been considered, relabel child-versions of extension and defeat colors to their expired-versions: $\rho_{ec^o_c\mapsto ex_c}, \rho_{dc^o_c\mapsto dx_c}$ for any $o\in\{+,-\}$.
			Relabel (current) extension and defeat colors to their child-version: $\rho_{e^o_c\mapsto ec^o_c}, \rho_{d^o_c\mapsto dc^o_c}$ for $o\in\{+,-\}$.
		\end{enumerate}
		\item \textbf{for relabeling} $b=\rho_{c'\mapsto c}$, we similarly repeat steps $1$--$9$ as before for corresponding clauses.
		%
		\begin{enumerate}
			\item consider a clause $C_{b,c}$  corresponding to Formula~(\ref{stab:relabel}) for the given operation $b$ and color $c\in\colors(b)$. Initiate $e^o_c(e_c^b), cm(C_{b,c})$ where $o=+/-$ depending on whether $C_{b,c}$ contains $e_c^b$ or $\neg e_c^b$, respectively. 
			Moreover, relabel (if not already) the color of $e_c^{b'}$ from the child operation $b'$ so that we have: $ec^p_c(e_c^{b'})$ where $p=+$ if $C_{b,c}$ contains $e_c^b$, and $p=-$ if it contains $\neg e_c^b$.
			\item take disjoint union ($\oplus$) of operation used in 1 above. 
			\item add 
			edges for the considered clause $C_{b,c}$ of Formula~(\ref{stab:relabel}): $\eta_{e^o_c,cm}$ (resp., $\eta_{cm, e^o_c}$) and $\eta_{cm, ec^{p}_c}$ ($\eta_{ec^{p}_c,cm}$) accordingly, depending on the current colors $o$ and $p$ of literals in 1 above.
			\item relabel $C_{b,c}$ as completed: $\rho_{cm\mapsto cr}$
			\item consider a clause $C_{b,c}$ corresponding to Formula~(\ref{stab:defeat-relabel}) for given $b$ and $c$.
			Initiate $d^o_c(d_c^b)$ and $cm(C_{b,c})$ where $o=+/-$ depending on whether $C_{b,c}$ contains $d_c^b$ or $\neg d_c^b$, respectively. 
			Moreover, relabel (if not already) the color of $d_c^{b'}$ from the child operation $b'$ so that we have: $dc^p_c(d_c^{b'})$ where $p=+/-$ depending on whether $C_{b,c}$ contains $d_c^{b'}$ or $\neg d_c^{b'}$, respectively.
			\item take disjoint union ($\oplus$) of operations form 4 and 5 above.
			\item add 
			edges for the considered clause $C_{b,c}$ of Formula~(\ref{stab:defeat-relabel}): $\eta_{d^o_c,cm}$  (resp., $\eta_{cm, d^o_c}$) and $\eta_{cm, dc^p_c}$ ( $\eta_{dc^p_c, cm}$) accordingly, depending on the current colors $o$ and $p$ of literals in 6 above.
			\item relabel $C_{b,c}$ as completed: $\rho_{cm\mapsto cr}$.
			\item Re-iterate for any remaining clauses corresponding to Formulas~(\ref{stab:relabel}) or (\ref{stab:defeat-relabel}). Once all clauses have been considered, relabel child-versions of extension and defeat colors to their expired-versions: $\rho_{ec^o_c\mapsto ex_c}, \rho_{dc^o_c\mapsto dx_c}$ for any $o\in\{+,-\}$.
			Relabel (current) extension and defeat colors to their child-version: $\rho_{e^o_c\mapsto ec^o_c}, \rho_{d^o_c\mapsto dc^o_c}$ for $o\in\{+,-\}$.
		\end{enumerate}

		\item \textbf{for edge introduction} $b= \eta_{c',c}$: we again repeat steps. 
		\begin{enumerate}
			\item consider a clause $C_{b,c}$  corresponding to Formula~(\ref{stab:edge}) for the given operation $b$ and color $c\in\colors(b)$. Initiate $e^o_c(e_c^b), cm(C_{b,c})$ where $o=+/-$ depending on whether $C_{b,c}$ contains $e_c^b$ or $\neg e_c^b$, respectively. 
			Moreover, relabel (if not already) the color of $e_c^{b'}$ from the child operation $b'$ so that we have: $ec^p_c(e_c^{b'})$ where $p=+$ if $C_{b,c}$ contains $e_c^b$, and $p=-$ if it contains $\neg e_c^b$.
			\item take disjoint union ($\oplus$) of operation used in 1 above. 
			\item add 
			edges for the considered clause $C_{b,c}$ of Formula~(\ref{stab:edge}): $\eta_{e^o_c,cm}$ (resp., $\eta_{cm, e^o_c}$) and $\eta_{cm, ec^{p}_c}$ ($\eta_{ec^{p}_c,cm}$) accordingly, depending on the current colors $o$ and $p$ of literals in 1 above.
			\item relabel $C_{b,c}$ as completed: $\rho_{cm\mapsto cr}$
			\item consider a clause $C_{b,c}$ corresponding to Formula~(\ref{stab:defeat-edge}) for given $b$ and $c$.
			Initiate $d^o_c(d_c^b)$ and $cm(C_{b,c})$ where $o=+/-$ depending on whether $C_{b,c}$ contains $d_c^b$ or $\neg d_c^b$, respectively. 
			Moreover, relabel (if not already) the color of $d_c^{b'}$ from the child operation $b'$ so that we have: $dc^p_c(d_c^{b'})$ where $p=+/-$ depending on whether $C_{b,c}$ contains $d_c^{b'}$ or $\neg d_c^{b'}$, respectively.
			In this case, we also need to consider variables $e_{c'}^b$.
			Thus, we relabel (if not already) the color of $e_{c'}^{b}$ so that we have: $e^p_{c'}(e_{c'}^b)$ where $p=+/-$ depending on whether $C_{b,c}$ contains $e_{c'}^b$ or $\neg e_{c'}^b$, respectively
			\item take disjoint union ($\oplus$) of operations form 4 and 5 above.
			\item add 
			edges for the considered clause $C_{b,c}$ of Formula~(\ref{stab:defeat-edge}): $\eta_{d^o_c,cm}$  (resp., $\eta_{cm, d^o_c}$) and $\eta_{cm, dc^p_c}$ ( $\eta_{dc^p_c, cm}$) accordingly, depending on the current colors $o$ and $p$ of literals in 6 above.
			\item relabel $C_{b,c}$ as completed: $\rho_{cm\mapsto cr}$.
			\item Re-iterate for any remaining clauses corresponding to Formulas~(\ref{stab:edge}) or (\ref{stab:defeat-edge}). Once all clauses have been considered, relabel child-versions of extension and defeat colors to their expired-versions: $\rho_{ec^o_c\mapsto ex_c}, \rho_{dc^o_c\mapsto dx_c}$ for any $o\in\{+,-\}$.
			Relabel (current) extension and defeat colors to their child-version: $\rho_{e^o_c\mapsto ec^o_c}, \rho_{d^o_c\mapsto dc^o_c}$ for $o\in\{+,-\}$.
		\end{enumerate}
		
		\item In the root operation, we only need to initiate clauses $C_c$ due to Formulas~(\ref{stab:defeat-root}) and add 
		the corresponding single edges from these clauses to $d_c$.
	\end{itemize}
	The correctness of the $k'$-expression follows from the fact that (I.) each clause from $\varphi_\CStab$ is eventually added, 
	(II.) only edges due to the clauses in $\varphi_\CStab$ are added,
	and (III.) the direction of each edge correctly depicts the membership of a literal in the clause. 
	Observe that (III.) holds since we add an edge $C{\rightarrow }x$ 
	if $x\in C$ and $C{\leftarrow }x$
	if $\neg x\in C$.
	Corresponding to each color $c$, we use five copies used as extension versions, given as $\{e_c^+, e_c^-\}\cup\{ec_c^+, ec_c^-\}\cup\{ex_c\}$.
	Likewise, we use five copies of defeat versions of each color.
	Moreover, two additional colors are required for adding 
	edges between clauses and literals.
	We conclude by observing that, one requires $5k+5k$ additional colors for drawing the incidence graph of the formula.
	This results in a requirement of $k'=11k+2$ colors to add 
	the expression $\ptlong'$, hence an increase of clique-width which is still linear in $k$.
\end{proof}

\paragraph{Admissible Extensions}

\admcorrect*
\begin{proof}
	``$\Longrightarrow$'': Let $S\in \adm(F)$. 
	As before, we prove that $\varphi_\CAdm$ is satisfiable by constructing an assignment $\alpha$ over variables $E$ and $A$ in a bottom-up fashion.
	The evaluation for variables in $E$ remains as in the proof of Theorem~\ref{thm:stab:correct} before.

	Regarding the attacking variables, we construct $\alpha$ as follows.
	For every color $c$ and operation $b$ of the $k$-expression, we set $\alpha(a_c^b)=0$ iff either
	(A1) $b$ is an initial operation with $c\in\colors(b)$; or
	(A2) $\alpha(a_c^{b'})=0$ for each $b'\in\child(b)$ s.t. $c\in\colors(b')$, if 
	$b$ is disjoint union operation; 
	or 
	(A3) $\alpha(a_{c'}^{b'})=0$ and $\alpha(a_c^{b'})=0$, if $b= \rho_{c'\mapsto c}$ with $c\in\colors(b)\cap \colors(b')$; or
	(A4a) $\alpha(a_c^{b'})=0$ and $\alpha(e_{c'}^{b})=0$, if $b= \eta_{c,c'}$ but not $\eta_{c',c}$ with $c\in \colors(b)$; 
	(A4b) either $\alpha(a_c^{b'})=0$ or $\alpha(e_{c'}^{b})=0$, if $b= \eta_{c',c}$ with $c\in \colors(b)$.

	We argue that $\alpha$ satisfies $\varphi_{\CAdm}$.
	Observe that  Formulas~(\ref{stab:leaf})--(\ref{stab:edge}) are satisfied by $\alpha$ due to the proof of Theorem~\ref{thm:stab:correct} and the fact that $S$ is conflict-free.
	Similarly, $\alpha$ satisfies Formulas~(\ref{adm:leaf})--(\ref{adm:edge-pos}) by construction.
	It remains to prove that $\alpha$ satisfies Formulas~(\ref{adm:root}) for each color $c$.
	Suppose to the contrary that there is some color $c$ such that $\alpha(a_c^{rt})=1$.
	Observe that $\alpha(a_c^b)=0$ for each initial operation $b$ initiating the color $c$ (due to Formulas~(\ref{adm:leaf})).
	Moreover, 
	Formulas~(\ref{adm:union})--(\ref{adm:relabel}) then merely propagate the values from child operation to their parents.
	However, $\alpha(a_c^{rt})=1$ implies that the value for $\alpha(a_c^b)$ changes at some edge introduction operation $b$ due to Formulas~(\ref{adm:edge-neg}).
	In particular, $\alpha(a_c^b)=1$ for at least one such operation $b$.
	Due to Formulas~(\ref{adm:edge-neg}), either $\alpha(a_c^{b'})=1$ for $b'\in\children(b)$, or $\alpha(e_{c'}^b)=1$ if $b= \eta_{c,c'}$.
	We consider the latter case, as the former case leads to repeating the same argument for $a_c^{b'}$.
	Since $\alpha(e_{c'}^b)=1$, this implies that $c'$ is an extension color in~$b$, hence there is some argument $a'$ of color $c'$ such that $a'\in S$.
	Moreover, due to the edge introduction $(c,c')$, there is an argument~$a$ of color~$c$ such that $(a,a')\in R$.
	Since $S$ is admissible, $a'$ must be defended against $a$, i.e., there is some argument $a''\in S$ with $(a'',a)\in R$.
	Suppose $a''$ be of color $c''$. 
	Due to $(a'',a)\in R$, there is an edge introduction operation $b''$ such that Formula~(\ref{adm:edge-pos}) applies to $c$ for introducing the edge $(c'',c)$.
	This implies that $\alpha(a_c^{b''})=0$ since $\alpha(e_{c''}^{b''})=1$ due to Formulas~(\ref{adm:edge-pos}) for $b''$ and $c$.
	Since $b$ was an arbitrary operation with $\alpha(a_c^b)=1$, this implies that {for each such $b$, we can find such an edge introduction operation due to the admissibility of $S$}.
	%
	Hence, we obtain a contradiction to $\alpha(a_c^b)=1$ and hence $\alpha(a_c^b)=0$ for any operation~$b$, in particular for~$b=rt$. 
	This completes our claim in this direction. Hence, $\alpha$ satisfies $\varphi_\CAdm$.
	
	To prove that $\alpha$ is unique for the set $S$, we 
	again consider another assignment $\beta\neq\alpha$ over $E\cup A$ that agrees with $\alpha$ over extension variables.
	However, as before, since $\beta$ assigns variable in $E$ exactly as $\alpha$, the same applies to variables in $A$ since $\beta$ has to satisfy Formulas~(\ref{out:leaf}) and (\ref{out:root}).
	This leads to a contradiction. Hence, for every admissible extension~$S$ in~$F$, there is exactly one satisfying assignment~$\alpha$ to variables~$E\cup A$.

	``$\Longleftarrow$'':
	Suppose there is an assignment $\theta$ over variables $E$ and $A$ such that $\theta\models \varphi_\CAdm$.
	Then, we construct an admissible extension $S=\{a\mid \theta(e_a)=1\}$ for $F$.
	$S$ is conflict-free due to Theorem~\ref{thm:stab:correct} since $\theta$ satisfies Formulas~(\ref{stab:leaf})--(\ref{stab:edge}).
	To prove admissibility, let $a'\in S$ be an argument.
	Suppose to the contrary, there is an argument $a\in A\setminus S$ such that $(a,a')\in R$, but there is no $a''\in S$ with $(a'',a)\in R$.
	Due to Formulas~(\ref{adm:root}), it follows that $\theta(a_c^{rt})=0$ for each $c\in\colors(rt)$.
	In particular, $\theta(a_{c}^{rt})=0$ for the color $c$ of argument $a$ in the root operation of the $k$-expression.
	We prove that this is not possible since the argument $a$ with $col(a)=c$ still attacks our extension $S$.
	Since $(a,a')\in R$, there is an edge introduction operation $b=\eta_{c,c'}$, wlog, since the similar argument applies if this happens before one of the colors were relabeled to $c$ or $c'$.
	However, then we have $\theta(e_{c'}^b)=1$ since $a'\in S$.
	But then, this leads to $\theta(a_c^b)=1$ due to Formulas~(\ref{adm:edge-neg}) since the edge $\eta_{c,c'}$ is being introduced instead of $\eta_{c',c}$.
	Moreover, there is no $a''\in S$ attacking $a$ therefore Formula~(\ref{adm:edge-pos}) does not apply to the color $c$ of $a$ in any operation. 
	Consequently, $\theta(a_c^{b})=1$ is propagated until the root operation and hence $\theta(a_c^{rt})=1$.
	However, this leads to a contradiction since $\theta$ satisfies Formulas~(\ref{adm:root}) for each $c$.
	
	\noindent	This concludes the claim and completes the correctness.
\end{proof}

\admcwaw*
\begin{proof}
	The proof follows analogous reasoning to the proof of Theorem~\ref{thm:stab:cwaw}.
	This is established by (1) replacing variables in $D$ by those in $A$, (2) renaming \emph{defeat} versions of each color ($d_c$) to be attack color ($a_c$) to fit our intuition, and (3) adding 
	edges corresponding to clauses arising from Formulas~(\ref{adm:leaf})--(\ref{adm:root}) instead of Formulas~(\ref{stab:defeat-leaf})--(\ref{stab:defeat-root}).
	It is easy to observe that one still needs $11k+2$ colors.
	As a result, the number of colors still increases only linearly.
\end{proof}

\paragraph{Complete Extensions}
\compcorrect*
%
\begin{proof}
	``$\Longrightarrow$'': Let $S$ be a complete extension in F.
	We prove that $\varphi_\CComp$ is satisfiable by constructing an assignment $\alpha$ over variables $E,A,O$ and $D$ in a bottom-up fashion.
	The evaluation for variables in $E$ and $A$ remains the same as in the proof of Theorem~\ref{thm:adm:correct}.

	Regarding the out $O$ and defeated $D$ variables, we construct $\alpha$ as follows.
	First, we set $\alpha(d_c^{\geq b'})=1$ iff
	(D1) either $\alpha(d_c^{\geq b})=1$ or $\alpha(e_{c'}^b)=1$ for any $b\neq \rho$ with $b'\in\children(b)$ and $c\in \colors(b')$; or
	(D2) $\alpha(d_{c'}^{\geq b})=1$ if $b=\rho_{c\mapsto c'}$ for $b'\in\children(b)$ and $c\in \colors(b)$,
	or if $\alpha(d_{c}^{\geq b})=1$ for the non-relabeling $c$.
	Then, for every color $c$ and operation $b$ of the $k$-expression, we set $\alpha(o_c^b)=1$ iff either
	(O1) $\alpha(e_c^b)=0$ for $b$ an initial operation with $c\in\colors(b)$; or
	(O2) $\alpha(o_c^{b'})=1$ for some $b'\in\child(b)$ s.t. $c\in\colors(b')$, if $b= \oplus$; or
	(O3) either $\alpha(o_c^{b'})=1$ if $c\in\colors(b')$ or $\alpha(o_{c'}^{b'})=1$, for $b= \rho_{c'\mapsto c}$ with $c\in\colors(b)$; or
	(O4a) $\alpha(o_c^{b'})=1$ and $\alpha(o_{c'}^{b})=0$, if $b= \eta_{c,c'}$ but not $\eta_{c',c}$ with $c\in \colors(b)$;
	(O4b) $\alpha(o_c^{b'})=1$ with $c\neq c'$ and additionally $\alpha(d_{c'}^{\geq b})=1$, if $b= \eta_{c',c}$ with $c\in \colors(b)$.
	
	We argue that $\alpha\models\varphi_{\Comp}$.
	Observe that  Formulas~(\ref{stab:leaf})--(\ref{stab:edge}) are satisfied by $\alpha$ due to the proof of Theorem~\ref{thm:stab:correct} and the fact that $S$ is  conflict-free whereas Formulas~(\ref{adm:leaf})--(\ref{adm:edge-pos}) due to the proof of Theorem~\ref{thm:adm:correct} and admissibility of $S$.
	Moreover, $\alpha$ satisfies Formulas~(\ref{out:leaf})--(\ref{def:rel}) by construction.
	It remain to prove that $\alpha$ satisfies Formulas~(\ref{out:root}) for each color $c$.
	The case for $\neg a_c^{rt}$ follows due to the proof of Theorem~\ref{thm:adm:correct} since $S$ is admissible.
	Therefore, we consider the only remaining case of $\neg o_c^{rt}$.
	To this aim, suppose to the contrary that there is some color $c$ such that $\alpha(o_c^{rt})=1$.
	
	Observe that this is only possible if $\alpha(e_c^b)=0$ for the initial operation $b$ initiating some argument $a$ of color $c$ (due to Formulas~(\ref{out:leaf})).
	This is without loss of generality since the case of a different color $c^*\neq c$ follows analogous reasoning due to the relabeling in Formulas~(\ref{out:relabel}).
	Since $\alpha(o_c^{rt})=1$, this implies that the value for $\alpha(o_c^b)$ must not change at any edge introduction operation $b$ due to Formulas~(\ref{out:att})--(\ref{out:defend}).
	Now, we look at the acceptance status of the argument $a$ with respect to $S$.
	Since $\alpha(e_c^b)=0$, we have that $\alpha(e_a)=0$ due to Formula~(\ref{stab:leaf}).
	Intuitively, $a$ is not in the extension but has been incorrectly left out.
	To obtain a contradiction, we prove that $a$ is defended by $S$.
	If $a$ is not attacked by any argument in $F$, then we get a contradiction straightforwardly since $a\in S$ must be true as $S$ is complete.
	Now, let $a'\in A$ be an arbitrary argument such that $(a',a)\in R$.
	Moreover, let $col(a')=c'$ in the operation $b$ when the attack $(a',a)$ is added.
	We consider the following two cases.
	
	(I). $a'\in S$. Then, $a$ can not be defended by $S$ since this would mean that there exists $s'\in S$ with $(s',a')\in R$, being a contradiction to $S$ being conflict-free.
	But this implies that $\alpha(d_{c'}^b)=0$ due to Formula~(\ref{def:union}) for the color $c'$ of $a'$ in $b$.
	However, this leads to a contradiction due to Formula~(\ref{out:defend}), since $\alpha(o_c^{b'})=0$ for every ancestor operation $b'$ of $b$, and in particular $\alpha(o_c^{rt})=0$.
	
	(II). $a'\not\in S$. Then, we again consider two sub-cases depending on whether $a'$ is attacked or not by $S$.
	The scenario when $a'$ is not attacked by $S$ follows the same reasoning as in case (I), since $c'$ is not defeated and hence $\alpha(o_c^{rt})=0$ leading to a contradiction.
	In the second scenario, if $a'$ is attacked by $S$, then $\alpha(d_{c'}^{\geq b})=1$ (due to Formula~(\ref{def:union}) since there is some argument $s\in S$ of color $c''$ and an edge introduction operation $(c'',c)$) and hence $\alpha(o_c^b)=1$ still remains true due to Formula~(\ref{out:defend}).
	However, since $a'$ was arbitrary, this is true for any $a'\in A$.
	As a result, the argument $a$ is actually defended by $S$, and hence $a\in S$ must be true since $S$ is a complete extension.
	But this leads to a contradiction to the fact that $\alpha(e_a)=0$ due to Formula~(\ref{stab:leaf}) and (\ref{out:leaf}).
	This completes our claim in this direction and hence $\alpha\models \varphi_\Comp$.
	
	The argument that $\alpha$ is unique for the set $S$ follows similar reasoning as in the proof of Theorem~\ref{thm:adm:correct}. Consider another assignment $\beta\neq\alpha$ over $E\cup A\cup O$ that agrees with $\alpha$ over extension variables.
	As before, since $\beta$ sets variable in $E$ exactly as $\alpha$, the same applies to variables in $A$ (due to satisfaction of Formulas~(\ref{out:leaf}) and (\ref{out:root})) and $O$ (due to  Formulas~(\ref{out:leaf}) and (\ref{out:root})).
	This leads to a contradiction and hence, for every complete extension $S$ in $F$, there is exactly one satisfying assignment $\alpha$ to variables $E\cup A$.
	
	``$\Longleftarrow$'':
	Suppose there is an assignment $\theta$ over variables $E,A,O$ and $D$ such that $\theta\models \varphi_\CComp$.
	Then, we construct a complete extension $S=\{a\mid \theta(e_a)=1\}$ for $F$.
	$S$ is clearly conflict free due to Theorem~\ref{thm:stab:correct} since $\theta$ satisfies Formulas~(\ref{stab:leaf})--(\ref{stab:edge}).
	Moreover, $S$ is admissible due to the proof of Theorem~\ref{thm:adm:correct} and Formulas~(\ref{adm:leaf})--(\ref{adm:root}).
	To prove that $S$ is in fact compete, suppose to the contrary, there is an argument $a\in A\setminus S$ such that for every $(a',a)\in R$, there is some $s\in S$ with $(s,a')\in R$.
	Due to Formulas~(\ref{out:root}), it follows that $\theta(o_c^{rt})=0$ for each $c\in\colors(rt)$, in particular, for the color $c$ of argument $a$. 
	We prove that this is not possible since the argument $a$ with $col(a)=c$ is incorrectly left out from $S$.
	
	Since $a\not\in S$, we have $\theta(e_a)=0$ as well as $\theta(e^b_c)= 0$ due to Formula~(\ref{stab:leaf}) for the initial operation $b$ for argument $a$ and color $c$.
	Hence, $\theta(o_c^b)=1$ due to Formula~(\ref{out:leaf}) for the same operation $b$ and color $c$.
	Since there is an argument $a'$ with $(a',a)\in R$, we look at the satisfaction of Formulas~(\ref{out:defend}) for the edge introduction operation $b=\eta_{c',c}$ where $a'$ has color $c'$ in b.
	Since, $(s,a')\in R$ for each attacker $a'$ of $a$, we have that $\theta(e^{b_s}_{c_s})=1$ for some edge introduction operation $b_s$ with $\colors(s)=c_s$ in $b_s$.
	As a result,  we have $\theta(d_{c'}^{\geq b})=1$ due to Formula~(\ref{def:union})--(\ref{def:rel}). 
	But this implies that, $\theta(o_c^{b})=1$ is true for this edge introduction operation $b$, and also propagated to every ancestor operation of $b$.
	We next consider the following two cases when this value can change, and prove that both lead to a contradiction.
	
	(I). In some edge introduction operation $b$, Formula~(\ref{out:att}) applies.
	Then, due to the edge introduction $b=\eta_{c,c'}$, we have $\theta(o_c^b)=0$ iff $\theta(e_{c'}^b)=1$.
	That is, $a$ attacks an argument $s'$ of color $c'$ and $s'\in S$. But this leads to a contradiction as $S$ can not defend $a$, being a conflict-free set.
	
	(II). In some edge introduction operation $b$, Formula~(\ref{out:defend}) applies.
	Then, due to the edge introduction $b=\eta_{c',c}$, we have $\theta(o_c^b)=0$ iff $\theta(d_{c'}^{\geq b})=0$. In other words, $a$ is attacked by an argument $a'$ of color $c'$ but not defended against it.
	This  again leads to a contradiction since $S$ defends $a$ as per our assumption.

	\noindent As a result, we must have $\theta(o_c^{rt})=1$.
	However, this contradicts to the fact that $\theta\models \varphi_\CComp$ (in particular, due to Formula~(\ref{out:root})).
	
	
	This concludes correctness and establishes the claim.
\end{proof}

\compcwaw*
\begin{proof}
	Following analogous reasoning to the proof of Theorem~\ref{thm:adm:cwaw}, we highlight the changes due to Formulas~(\ref{out:leaf})--(\ref{def:rt}).
	This is established by adding (1) a copy of variables in $A$ renamed to serve `out' variables $O$, (2) having \emph{out} versions of each colors ($o_c$) corresponding to each version of the attack color ($a_c$), and (3) adding 
	edges corresponding to clauses arising from Formulas~(\ref{out:leaf})--(\ref{out:root}) by simply adapting the same procedure as for Formulas~(\ref{adm:leaf})--(\ref{adm:root}).
	Finally, for each operation we need to add 
	clauses corresponding to Formulas~(\ref{def:rt})--(\ref{def:union}), which can be achieved by copying their clauses accordingly and using additional colors to model defeat variables.
	Observe that the additional colors are needed to add 
	clauses for out and defeat variables.
	It can be noticed that the number of colors required now doubles, compared to the case of admissible semantics.
	The additional $11k+2$ colors are needed due to the positive/negative and current/child/expired versions of attack and defeat colors used in the expression construction of $\varphi_\CComp$.
	As a result, the total number of colors being $2\cdot(11k+2)$ still increases only linearly.
\end{proof}

\paragraph{Meta Result on DNF Matrices}
\thmDNF*
\begin{proof}
	It suffices to encode the CNF $\varphi$ into a DNF formula along the k-expression. 
	To this end, we require (universally quantified) auxiliary variables of the form $sat_c^b$,
	indicating that $\varphi$ is satisfied for color $c$ up to $b$ as well as $t_c^b$ or $f_c^b$,
	which indicates that $c$ in $b$ has a variable that is true or false, respectively.
	We construct DNF $\neg\varphi'$ with $\varphi'$ given next:
	
	\begin{flalign}
		%
		%
		&
		sat_c^b, \;\, t_{c}^b \leftrightarrow \hspace{-.15em}v, \;\, f_{c}^b \leftrightarrow \hspace{-.15em}\neg v && \hspace{-2.5em}\text{initial $b$, variable $v$ of color $c$}\raisetag{3.15em}\\
		& \neg{}sat_c^b, \quad \neg t_{c}^b,\quad \neg f_c^b && \hspace{-3em}\text{initial $b$, clause $f$ of color $c$}\raisetag{1.05em}\\
		&sat_{c}^b \leftrightarrow \hspace{-3.5em}\bigwedge_{\substack{b'\in\children(b): c\in \colors(b')}} \hspace{-3.25em}sat_{c}^{b'},	&& \hspace{-3.5em}\text{disjoint union $b$, every $c\in \colors(b)$}\raisetag{2.45em}\notag\\[-.45em]
		&t_{c}^b \leftrightarrow \hspace{-2em}\bigvee_{\substack{b'\in\children(b): c\in \colors(b')}} \hspace{-3.25em}t_{c}^{b'}, \qquad f_{c}^b \leftrightarrow \hspace{-2em}\bigvee_{\substack{b'\in\children(b): c\in \colors(b')}} \hspace{-3.25em}f_{c}^{b'}\hspace{-20em}\\	
		&sat_c^b \leftrightarrow \hspace{-2.5em}\bigwedge_{\text{if $b$ relabeling\ }c'\mapsto c} \hspace{-2.5em} sat_{c'}^{b'}\hspace{-.25em}\wedge \hspace{-.1em}\bigwedge_{\text{if\ }c{\in} \colors(b')}\hspace{-1.75em} sat_{c}^{b'},
		&& \hspace{1em}\text{relabeling $b$, every $c\in$}\notag\\[-1.7em]
		&&&\hspace{1em} \colors(b), b'\in\children(b)\notag\\[-.35em]
		&t_c^b \leftrightarrow \hspace{-2.5em}\bigvee_{\text{if $b$ relabeling\ }c'\mapsto c} \hspace{-2.5em} t_{c'}^{b'}\hspace{-.25em}\vee \hspace{.85em}\bigvee_{\text{if\ }c{\in} \colors(b')}\hspace{-1.75em} t_{c}^{b'},
		\quad f_c^b \leftrightarrow \hspace{-2.5em}\bigvee_{\text{if $b$ relabeling\ }c'\mapsto c} \hspace{-2.5em} f_{c'}^{b'}\hspace{-.25em}\vee \hspace{.65em}\bigvee_{\text{if\ }c{\in} \colors(b')}\hspace{-1.75em} f_{c}^{b'}\hspace{-100em}\\
		&\hspace{-.1em}sat_{c}^b \leftrightarrow sat_{c}^{b'} \hspace{-.2em}\vee\hspace{-2em} \bigvee_{\substack{\text{if $b$ $+$edge}\\\text{introduce $(c',c)$}}}\hspace{-2em}t_{c'}^b \vee \bigvee_{\substack{\text{if $b$ $-$edge}\\\text{introduce $(c',c)$}}}\hspace{-2em}f_{c'}^b,\hspace{-10em}	&& \hspace{1em}\text{edge introduce $b$, every}
	\notag\\[-2.5em]
	&&&  \hspace{1.5em}c{\in} \colors(b), b'{\in}\children(b)\notag\\[.5em]
	&\hspace{-.1em}t_{c}^b \leftrightarrow t_{c}^{b'}, \quad f_{c}^b \leftrightarrow f_{c}^{b'} \hspace{-10em}	&& 
\end{flalign}
In the end, we require satisfiability for every color:
\begin{flalign}
\label{dnf:sat}{sat}_c^{rt} 	&& \text{for root colors\ } c\in \colors(rt)
\end{flalign}
The idea of this approach is to guide satisfiability of clauses ($sat_c^b$) along,
thereby keeping the information of whether a variable is assigned false ($f_c^b$)
or true ($t_c^b$).
Then, the resulting DNF matrix will be $\varphi'' = \neg\varphi' \wedge \psi$.
It is easy to see that the directed incidence clique-width is linearly preserved.
\end{proof}

\paragraph{Preferred Extensions}
\prefcorrect*
\begin{proof}
Observe that, following Theorem~\ref{thm:adm:correct}, there exists a one-one correspondence between admissible extension $S$ in $F$ and satisfying assignments $\alpha$ for $\varphi_\CAdm$.
Therefore, it suffices to prove that the Formulas~(\ref{pref:leaf})--(\ref{pref:root}) only allow subset-maximal admissible (hence preferred) extensions.

``$\Longrightarrow$''
Let $S$ be a preferred extension in $F$. Then, we construct a unique satisfying
assignment $\alpha$ for $\varphi_\CPref$.
We set, $\alpha(e_a)=1$ for every $a\in S$ and $\alpha(e_a)=0$ otherwise.
Furthermore, for remaining variables in $E$ and $A$, we map them exactly as in the proof of Theorem~\ref{thm:adm:correct}.
Then, $\alpha\models\varphi_\CAdm$ since $S$ is admissible.

Now, we prove that $\varphi_\CPref[\alpha]\equiv 1$.
Assume towards a contradiction that there exists an assignment $\beta$ over variables in $E^*,A^*,S$, such that $\varphi_\CPref[\alpha][\beta]\equiv 1$.
In particular, this implies that $\beta\models\varphi^*_\CAdm$ due to Formulas~(\ref{pref:adm}).
Moreover, there is some color $c$ such that $\beta(s_c^{rt})=1$. due to Formulas~\ref{pref:root}.
However, this implies that there is some initial operation $b$ such that $\beta(s_{c'}^b)=1$ for some color (possibly) $c'=c$ by Formulas~(\ref{pref:union})--(\ref{pref:edge}).
But this implies that $\beta(e_c^{* b})=1$ and  $\alpha(e_c^b)=0$ must be true.
We construct an admissible set $S^*$ via the assignment $\beta$ to variables in $E^*$. That is, we let $S^*=\{a\mid a\in F, \beta(e_a^*)=1\}$.
Then $S$ is indeed admissible since $\beta$ satisfies $\varphi_\CAdm^*$ due to Formulas~(\ref{pref:adm}).
Moreover, since $\varphi_\CPref[\alpha][\beta]\equiv 1$, the two assignments together satisfy Formulas~(\ref{pref:subset}) and hence $S\subseteq S^*\subseteq A$ by definition of $S^*$ and $\alpha$.
Furthermore, we have $S\subsetneq S^*$ due to Formula~(\ref{pref:leaf}).
However, this leads to a contradiction to $S$ being preferred since $S^*$ is also admissible.
Consequently, there is no such assignment $\beta$ satisfying Formulas~(\ref{pref:leaf})--(\ref{pref:root}).
Finally, observe that there is no assignment $\alpha'$ over variables $E\cup A$ such that $\alpha'\neq\alpha$ and $\alpha'$ is a satisfying assignment for $\varphi_\CPref$ if $\{a\mid a\in S\}=\{a\mid \alpha'(e_a)=1\}$. This holds due to the same claim for $\varphi_\CAdm$, since there is bijection between admissible extensions and satisfying assignments for $\varphi_{\CAdm}$.

``$\Longleftarrow$''.
Let $\alpha$ be a satisfying assignment over $E\cup A$ such that $\varphi_\CPref[\alpha]\equiv 1$.
Then, we construct an extension $S=\{a\mid a\in A,\alpha(e_a)=1\}$ and prove that $S$ is a preferred in $F$.
$S$ is indeed admissible since $\alpha\models \varphi_\CAdm$.
Now, assume towards a contradiction and let $S'$ be another admissible set in $F$ such that $S'\supset S$.
Then, we construct an assignment $\beta$ using $S'$ such that $\varphi_\CPref[\alpha][\beta]\equiv 1$.
For $a\in A$, we define $\beta(s_c^b)=1$ if there is an argument $a\in A$ in the leaf operation of our $k$-expression with color $c$ such that $a\in S'\setminus S$, and we set $\beta(s_c^b)=0$, otherwise.
Furthermore, we for each non-leaf $b$, we set $\beta(s_c^b)=1$ iff either (1) $\beta(s_c^{b'})=1$ for $b'\in \child(b)$, if $b$ is disjoint union, or (2) either $\beta(s_c^{b'})=1$ or $\beta(s_{c'}^{b'})=1$ for $b'\in \child(b)$, if $b$ is relabeling operation with for $c'\mapsto c$, or (3) $\beta(s_c^{b'})=1$ for $b'\in \child(b)$, if $b$ edge introduction operation.
Moreover, we define the evaluation of $\beta$ over remaining variables in $E^*\cup A^*$ in the same manner as in the poof of Theorem~\ref{thm:adm:correct} via mapping the given admissible set to a satisfying assignment $\alpha$ for $\varphi_{\CAdm}$.
Then, clearly $\beta$ satisfies Formulas~(\ref{pref:leaf})--(\ref{pref:edge}) by construction.
Finally, $\beta$ satisfies Formula~\ref{pref:root} since $\beta(s_c^b)=1$ in at initial leaf operation $b$ due to the existence of an argument $a\in S'\setminus S$.
But this leads to a contradiction to $\varphi_\CPref[\alpha]\equiv 1$ since there can be no such assignment $\beta$.
Consequently, for every satisfying assignment $\alpha$ to variables $E\cup A$ of $\varphi_\CPref$, $F$ admits exactly one preferred extension $S$.
\end{proof}

\prefcwaw*
\begin{proof}
Building on top of Theorem~\ref{thm:adm:cwaw}, we highlight the changes due to Formulas~(\ref{pref:adm})--(\ref{pref:root}).

First of all, we need $k'= 11k+2$ colors to construct a $k'$-expression for $\varphi_\CAdm$.
Then, the $k'$-expression for $\varphi^*_\CAdm$ requires the same number of colors, so that we can draw the incidence graph for $\varphi_\CAdm$ and $\varphi^*_\CAdm$ at the same time. 
Finally, to cover Formulas~(\ref{pref:leaf})--(\ref{pref:root}), we require a fresh color with all five versions as before, called \emph{inequality} colors, denoted as $\{i^o_c, ic^o_c\}\cup \{ix_c\}$ for $o\in \{+,-\}$.
Allowing different colors corresponding to variables in $\varphi_{\CAdm}$ and $\varphi^*_{\CAdm}$ had the advantage that clauses corresponding to Formulas~(\ref{pref:leaf}) and (\ref{pref:subset}) are also covered.
This is the case since literals $e_c^b$ and $e_c^{*b}$ both can be colored differently, and hence the edges for corresponding clauses can be added.

Consequently, $2(11k+2)$ colors are needed due to $\varphi_{\CAdm}$ and $\varphi^*_{\CAdm}$, and additional $5k$ colors are required to cover Formulas~(\ref{pref:leaf})--(\ref{pref:root}).
As a result, the total number of colors being $k^p=27k+4$ still increases only linearly in the construction of the $k^p$-expression $\ptlong^p$ for $\varphi_\CPref$.
\end{proof}

\paragraph{Semi-Stable Extensions}

\stagcorrect*
\begin{proof}
Once again, we begin by observing a one-one correspondence between admissible extension $S$ in $F$ and satisfying assignments $\alpha$ for $\varphi_\CAdm$ due to  Theorem~\ref{thm:adm:correct}.
Therefore, it suffices to prove that the Formulas~(\ref{semi:leaf})--(\ref{semi:root}) together with Formulas~(\ref{stab:defeat-leaf})--(\ref{stab:defeat-edge}) allow range-maximal admissible (hence semi-stable) extensions.

``$\Longrightarrow$''
Let $S$ be a semi-stable extension in $F$. Then, we construct a unique satisfying
assignment $\alpha$ for $\varphi_\CSemiSt$.
We set, $\alpha(e_a)=1$ for every $a\in S$ and $\alpha(e_a)=0$ otherwise.
Furthermore, for remaining variables in $E$ and $A$, we map them exactly as in the proof of Theorem~\ref{thm:adm:correct}.
Then, $\alpha\models\varphi_\CAdm$ since $S$ is admissible.
Finally, we map $\alpha$ to variables in $D$ as in the proof of Theorem~\ref{thm:stab:correct}. That is, we set $\alpha(d_c^b)=1$ for every color $c$
and operation~$b$ of the $k$-expression, if and only if,
(1) $\alpha(e_c^b)=1$, if $b=c(a)$ is an initial $k$-graph with $c\in\colors(b)$;
(2) $\alpha(d_c^{b'})=1$ for each~$b'\in\child(b)$
s.t. $c\in\colors(b')$, if $b= \oplus$;
(3) $\alpha(d_{c'}^{b'})=1$ and $\alpha(d_c^{b'})=1$, if $b= \rho_{c'\mapsto c}$ with $c\in\colors(b)\cap \colors(b')$; or
(4) either $\alpha(d_c^{b'})=1$ or $\alpha(e_{c'}^{b})=1$, if $b= \eta_{c',c}$ with $c\in \colors(b)$.
Then, $\alpha$ satisfies Formulas~(\ref{stab:defeat-leaf})--(\ref{stab:defeat-edge}) by construction.
Now, we prove that $\varphi_\CSemiSt[\alpha]\equiv 1$.
Assume towards a contradiction that there exists an assignment $\beta$ over variables in $E^*,A^*, D^*S$, such that $\varphi_\CSemiSt[\alpha][\beta]\equiv 1$.
In particular, this implies that $\beta\models\varphi^*_\CAdm$ due to Formulas~(\ref{semi:adm}).
Moreover, there is some color $c$ such that $\beta(s_c^{rt})=1$. due to Formulas~(\ref{semi:root}).
However, this implies that either (1) there is some initial operation $b$ such that $\beta(s_{c'}^b)=1$ by Formulas~(\ref{semi:union})--(\ref{semi:relabel}), or (2) there is some edge introducing operation $b$ such that $\beta(d_{c'}^{*b})=1$ but $\alpha(d_{c'}^{b})=0$ by Formulas~(\ref{semi:edge}), for some color (possibly) $c'=c$.
But both cases implies that $\beta(d_c^{* b})=1$ and  $\alpha(d_c^b)=0$ must be true for at least one color $c$.
We construct an admissible set $S^*$ via the assignment $\beta$ to variables in $E^*$ by letting $S^*=\{a\mid a\in F, \beta(e_a^*)=1\}$.
Then $S$ is indeed admissible since $\beta$ satisfies $\varphi_\CAdm^*$ due to Formulas~(\ref{semi:adm}).
Moreover, since $\varphi_\CSemiSt[\alpha][\beta]\equiv 1$, the two assignments together satisfy Formulas~(\ref{semi:root}) and hence $S^+_R\subseteq S^{*+}_R\subseteq A$ by definition of $S^*$ and $\alpha$.
Furthermore, we have $S^+_R\subsetneq S^{*+}_R$ due to Formula~(\ref{semi:leaf}).
However, this leads to a contradiction to $S$ being semi-stable since $S^*$ is also admissible and $S^{*+}_R$ has more arguments than $S^{+}_R$.
Consequently, there is no such assignment $\beta$ satisfying Formulas~(\ref{semi:leaf})--(\ref{semi:root}).
Finally, observe that there is no assignment $\alpha'$ over variables $E\cup A$ such that $\alpha'\neq\alpha$ and $\alpha'$ is a satisfying assignment for $\varphi_\CSemiSt$ if $\{a\mid a\in S\}=\{a\mid \alpha'(e_a)=1\}$. This holds due to the same claim for $\varphi_\CAdm$, since there is bijection between admissible extensions and satisfying assignments for $\varphi_{\CAdm}$.

``$\Longleftarrow$''.
Let $\alpha$ be a satisfying assignment over $E\cup A$ such that $\varphi_\CSemiSt[\alpha]\equiv 1$.
Then, we construct an extension $S=\{a\mid a\in A,\alpha(e_a)=1\}$ and prove that $S$ is a semi-stable in $F$.
$S$ is indeed admissible since $\alpha\models \varphi_\CAdm$.
Now, assume towards a contradiction and let $S'$ be another admissible set in $F$ such that $S^{'+}_R\supset S^+$.
Then, we construct an assignment $\beta$ using $S'$ such that $\varphi_\CSemiSt[\alpha][\beta]\equiv 1$.
For $a\in A$, we define $\beta(s_c^b)=1$ if there is an argument $a\in A$ in the leaf operation of our $k$-expression with color $c$ such that $a\in S^{'+}_R\setminus S^+_R$, and we set $\beta(s_c^b)=0$ otherwise.
Furthermore, we for each non-leaf $b$, we set $\beta(s_c^b)=1$ iff either (1) $\beta(s_c^{b'})=1$ for $b'\in \child(b)$, if $b$ is disjoint union, or (2) either $\beta(s_c^{b'})=1$ or $\beta(s_{c'}^{b'})=1$ for $b'\in \child(b)$, if $b$ is relabeling operation with for $c'\mapsto c$, or (3) either $\beta(s_c^{b'})=1$, or $\beta(d_c^{*b})=1$ and $\alpha(d_c^b)=0$ for $b'\in \child(b)$, if $b$ edge introduction operation.
Moreover, we define the evaluation of $\beta$ over remaining variables in $E^*\cup A^*$ (resp., $D^*$) in the same manner as in the poof of Theorem~\ref{thm:adm:correct} (Theorem~\ref{thm:stab:correct}). 
Then, clearly $\beta$ satisfies (\ref{stab:defeat-leaf})--(\ref{stab:defeat-edge}) and  Formulas~(\ref{semi:leaf})--(\ref{semi:edge}) by construction.
Finally, $\beta$ satisfies Formula~\ref{semi:root} since $\beta(s_c^b)=1$ in either some initial leaf operation or in some edge introduction operation $b$, due to the existence of an argument $a\in S{'^+}_R\setminus S^+_R$.
But this leads to a contradiction to $\varphi_\CSemiSt[\alpha]\equiv 1$ since there is no such assignment $\beta$.
Consequently, for every satisfying assignment $\alpha$ to variables $E\cup A$ of $\varphi_\CSemiSt$, $F$ admits exactly one preferred extension $S$.
\end{proof}

\stagcwaw*
\begin{proof}
The proof is analogous to the proof of Theorem~\ref{thm:pref:cwaw} and builds on top of Theorem~\ref{thm:adm:cwaw}.
We again highlight the changes due to Formulas~(\ref{stab:defeat-leaf})--(\ref{stab:defeat-edge}) and  Formulas~(\ref{semi:adm})--(\ref{semi:root}).
As before, we need $k'= 2(11k+2)$ colors to construct a $k'$-expression for $\varphi_\CAdm$ and $\varphi^*_\CAdm$.
To cover Formulas~(\ref{semi:leaf})--(\ref{semi:root}), we require fresh \emph{inequality} colors $\{i^o_c, ic^o_c\}\cup \{ix_c\}$ for $o\in \{+,-\}$.
However, now we need additional colors to cover \emph{defeated} variables in Formulas~(\ref{stab:defeat-leaf})--(\ref{stab:defeat-edge}), similar to the proof of Theorem~\ref{thm:stab:correct}.
Consequently, we additionally require colors $\{d^o_c, dc^o_c\}\cup \{dx_c\}$ for $o\in \{+,-\}$ as in the case of stable semantics.
Since we use different colors corresponding to variables in $\varphi_{\CAdm}$ and $\varphi^*_{\CAdm}$, as well as for variables in $D$ and $D^*$, the Formulas~(\ref{semi:leaf}) and (\ref{semi:root}) are all covered.

Consequently, $2(11k+2)$ colors are needed due to $\varphi_{\CAdm}$ and $\varphi^*_{\CAdm}$, and additional $2(5k)$ are required to cover Formulas~(\ref{semi:leaf})--(\ref{semi:root}).
As a result, the total number of colors being $k^s=32k+4$ still increases only linearly in the construction of the $k^s$-expression $\ptlong^s$ for $\varphi_\CSemiSt$.
\end{proof}

\paragraph{Lower Bounds with Cliquewidth}

\admlb*
\begin{proof}
We reduce from 3SAT by adapting a known reduction~\cite{DvorakWoltran10}.
For every variable in $\varphi$ we create arguments $x$ and $\overline{x}$
and attacks $x\rightarrowtail \overline{x}$, $x\leftarrowtail \overline{x}$.
Also, we create an argument $sat$ and an argument $c$
for every clause $c$, as well as attacks $c\rightarrowtail sat$.
For every literal $l$ in $c$, we add an attack $l\rightarrowtail{}c$.
The construction ensures that if there is an admissible extension containing
$sat$, all clauses need to be attacked, which works if we choose
arguments that are consistent, satisfying $\varphi$.

Observe that the construction is such that $F$ is indeed almost the directed incidence graph of $\varphi$.
Indeed, for any given k-expression, we can double all colors,
where the copy is for arguments $\overline{x}$ representing negated variables.
Then, the attack / directed edge from $l\rightarrowtail{}c$ can be easily created
whenever the corresponding directed edge is created.
Also, $x\rightarrowtail \overline{x}$, $x\leftarrowtail \overline{x}$ can be easily covered in the leaves.
Finally, $c\rightarrowtail sat$ can be covered in a leaf operation, where $c$ is introduced,
as we can easily just add an additional color for $sat$ in every operation of the k-expression.
Consequently, since the clique-width of $F$ only linearly increases from $w$, the
result holds.
\end{proof}

\preflb*
\begin{proof}[Proof (Idea)]
For preferred extensions, we modify a construction for skeptical acceptance of an argument~\cite{DunneBench-Capon02a},
reducing from $\forall\exists$-QBF.
For
semi-stable and stage extensions, we modify a construction for credulous acceptance of an argument~\cite{DvorakWoltran10}, which reduces from $\forall\exists$-QBF (and inverts).
Both constructions are quite similar and slight modifications of the construction used in the proof of Theorem~\ref{thm:lb}, but those do not significantly change the structure.
Indeed, we can double the number of used colors (for auxiliary arguments per universal argument variable)
and add a constant number of additional colors.
The remainder works as in Theorem~\ref{thm:lb}.
\end{proof}

\end{document}